\documentclass{article}

\usepackage[final]{neurips_2021}

\usepackage[utf8]{inputenc} 
\usepackage[T1]{fontenc}    
\usepackage{hyperref}       
\usepackage{url}            
\usepackage{booktabs}       
\usepackage{amsfonts}       
\usepackage{nicefrac}       
\usepackage{microtype}      
\usepackage{xcolor}         
\usepackage{longtable}
\usepackage{booktabs}
\usepackage{algorithmic}
\usepackage{graphicx}
\usepackage{subfigure}
\usepackage{textcomp}
\usepackage{booktabs}
\usepackage{multirow}
\usepackage{wrapfig}
\usepackage{graphicx}
\usepackage{amsmath,amscd,amsbsy,amssymb,amsfonts,latexsym,url,bm,amsthm}
\def\BibTeX{{\rm B\kern-.05em{\sc i\kern-.025em b}\kern-.08em
		T\kern-.1667em\lower.7ex\hbox{E}\kern-.125emX}}
\usepackage[vlined,boxed,commentsnumbered,linesnumbered,ruled]{algorithm2e}

\usepackage{threeparttable}

\newtheorem{theorem}{Theorem}

\newcommand{\eqword}{\fontsize{8pt}{18pt}\selectfont}

\title{Towards Open-World Feature Extrapolation: \\ An Inductive Graph Learning Approach}

\author{%
  Qitian Wu, Chenxiao Yang, Junchi Yan\thanks{Junchi Yan is the corresponding author.} \\
  Department of Computer Science and Engineering\\
  Shanghai Jiao Tong University\\
  \texttt{\{echo740, chr26195, yanjunchi\}@sjtu.edu.cn} \\
}

\begin{document}

\maketitle

\begin{abstract}
    We target open-world feature extrapolation problem where the feature space of input data goes through expansion and a model trained on partially observed features needs to handle new features in test data without further retraining. The problem is of much significance for dealing with features incrementally collected from different fields.
    To this end, we propose a new learning paradigm with graph representation and learning. Our framework contains two modules: 1) a backbone network (e.g., feedforward neural nets) as a lower model takes features as input and outputs predicted labels; 2) a graph neural network as an upper model learns to extrapolate embeddings for new features via message passing over a feature-data graph built from observed data. Based on our framework, we design two training strategies, a self-supervised approach and an inductive learning approach, to endow the model with extrapolation ability and alleviate feature-level over-fitting. We also provide theoretical analysis on the generalization error on test data with new features, which dissects the impact of training features and algorithms on generalization performance. Our experiments over several classification datasets and large-scale advertisement click prediction datasets demonstrate that our model can produce effective embeddings for unseen features and significantly outperforms baseline methods that adopt KNN and local aggregation. The implementation codes are public available at \url{https://github.com/qitianwu/FATE}.
\end{abstract}

\section{Introduction}
Learning a mapping from observation $\mathbf x$ (a vector of attribute features) to label $y$ is a fundamental and pervasive problem in ML community, with extensive applications spanning from classification/regression tasks for tabular data to advertisement click prediction \cite{feature-ctr, lr, wd, fm, fwl}, item recommendation \cite{feature-rec, feature-rec2, youtube, pinsage}, question answering \cite{qa2, qa1}, or AI more broadly. Existing approaches focus on a fixed input feature space shared by training and test data. Nevertheless, practical ML systems interact with a dynamic open-world where features are incrementally collected. For example, in recommender/advertisement systems, there often occur new user profile features unseen before for current prediction tasks. Also, with the advances of multi-modal recognition \cite{mm-4} and federated learning \cite{mm-2}, it is a requirement for a model trained with partial features to incorporate new features from other fields for decision-making on a target task.

A challenge stems from the fact that off-the-shelf neural network models cannot deal with new features without re-training on new data. As shown in Fig.~\ref{fig-intro}(a), a neural network builds a mapping from input features to a hidden representation through a weight matrix in the first layer. Given new features as input, the network needs to be augmented with new weight parameters and cannot map the new features to a desirable position in latent space if without re-training.

However, model re-training would be tricky and bring up several issues. First, re-training a model with both previous and new features would be highly time-consuming and cannot meet the requirement for online systems. Alternatively, one can re-train a trained model only on new features, which will induce risks for over-fitting new data or forgetting previous data.

Different from machines, fortunately, humans are often equipped with the ability for extrapolating to unseen features and distill the knowledge in new information for solving a target task without any re-training on new data. The inherent gap between existing ML approaches and human intelligence raises a research question: \emph{Can we design a ML model that is trained on one set of features and able to generalize to combine new unseen features for the same task without further training?} 


\begin{figure}[t]
    \centering
    \includegraphics[width=\textwidth,angle=0]{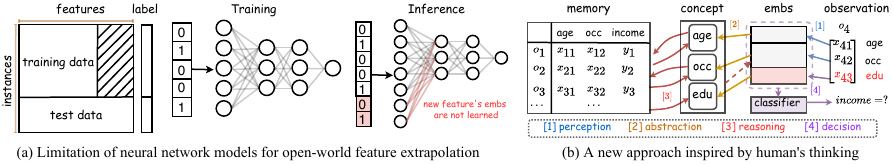}
    \vspace{-15pt}
    \caption{(a) The open-world feature extrapolation problem is defined over training data with partial features and test data with augmented feature space. Existing neural models cannot deal with new features without re-training. (b) Oue approach has an analogy to thinking process in human's brain. Imagine a person trained on using features \emph{age}, \emph{occ} to predict one's \emph{income} and tested on test cases with a new feature \emph{edu}. There are four processes from acquiring observation to giving final prediction.}
    \label{fig-intro}
    \vspace{-15pt}
\end{figure}

To find desirable solutions to this non-trivial question is challenging. Let us resort to how humans think and act in physical world scenarios. Imagine that we are well trained on predicting a person's income with one's age and occupation from historical records. Now we are asked to estimate the income of a new person with his/her age, occ. and education (a new feature). We can modularize the process in brain systems that distill and leverage the knowledge in new observations into four steps, as shown in Fig.~\ref{fig-intro}(b). 1) \emph{Perception}: the observations are recognized by our senses; 2) \emph{Abstraction:} the perceived information is aligned with concepts in our cognition; 3) \emph{Reasoning:} we search similar concepts and observations in our memory to understand and absorb the new knowledge; 4) \emph{Decision:} with new understanding and abstraction, we make final decisions for the prediction.

The above human's thinking process inspires the methodology in our paper where we propose a new learning paradigm for open-world feature extrapolation. Our proposed framework contains two modules: a backbone network, which can be a feedforward neural network, and a graph neural network. 1) The backbone network first maps input features to embeddings, which can be seen as \emph{perception} from observation. 2) We then treat observed data matrix (each row represents a feature vector for one instance) as a feature-data bipartite graph, which explicitly define proximity and locality structures among features and instances. Then a graph neural network is harnessed for neural message passing over adjacent features and instances in a latent space (\emph{abstraction}). 3) The GNN will inductively compute embeddings for new features based on those of existing ones, mimicking the \emph{reasoning} process from familiar concepts to new ones in our brain. 4) The newly obtained embeddings that capture both semantics and feature-level relations will be used to obtain hidden representations of new data with unseen features and make final \emph{decisions}.

To endow the model with ability for extrapolation to new features, we propose two training algorithms, one by self-supervised learning and one by inductive learning. The proposed model and its learning algorithm can easily scale to large-scale dataset (with millions of features and instances) via mini-batch training. We also provide a theoretical analysis on the generalization error on test data with new features, which allows us to dissect the impact of training features and algorithms on generalization performance. To verify our approach, we conduct extensive experiments on several real-world datasets, including six classification datasets from biology, engineering and social domains with diverse features, as well as two large-scale datasets for advertisement click prediction. Our model is trained on training data using partial features and tested on test data with a mixture of seen and unseen features.
The results demonstrate that 1) our approach consistently outperform models not using new features for inference; 2) our approach achieves averagely $29.8\%$ higher Accuracy than baseline methods using KNN, pooling or local aggregation for feature extrapolation; 3) our approach even exceeds models using incremental training on new features, yielding average $4.1\%$ higher Accuracy.

\textbf{Our contributions are:} 1) We formulate open-world feature extrapolation problem and show that it is feasible to extend neural models for extrapolating to new features without re-training; 2) We propose a new graph-learning model and two training algorithms for feature extrapolation problem; 3) Our theoretical analysis shows that the generalization error for data with new features relies on the number of training features and the randomness in training algorithms; 4) We conduct comprehensive experiments and show the effectiveness, applicality and scalability of proposed method.

\section{Methodology}

We focus on attribute features as input in this paper. An input instance is a vector $\mathbf r_i =[r_{im}]_{m=1}^d \in \mathbb R^d$ where each entry $r_{im}$ denotes a \emph{raw feature} (like age, occ., edu., etc.). If $r_{im}$ is a discrete/categorical raw feature with $R_m$ possible values, its space is an integer set $\{0, 1, \cdots, R_m-1\}$. If $r_{im'}$ is a continuous one, a common practice is to convert it into a discrete feature within space $\{0, 1, \cdots, R_{m'}-1\}$ by evenly division \cite{ffm} or log transformation \cite{feature-ctr}. We call $R_m$ as \emph{cardinality} for $m$-th raw feature.

An effective way to handle attribute features is via one-hot encoding \cite{wd,ffm,deepfm}. For $r_{im}$ with cardinality $R_m$ we convert it into a $R_m$-dimensional one-hot vector $\mathbf x_i^m$ where the unique 1 indexes the value.
In this way, one can convert an input $\mathbf r_i$ into a concatenation of one-hot vectors:
\begin{equation}\label{eqn-onehotvec}
    \mathbf x_i = \left[\mathbf x_i^{1}, \mathbf x_i^2, \cdots, \mathbf x_i^d\right], \;\mbox{where}\; 1\leq \forall m\leq d, \;\mathbf x_i^{m} \in \{0, 1\}^{R_m}\; \mbox{is a one-hot vector}.
\end{equation}
Use $x_{ij}$ to denote the $j$-th entry of $\mathbf x_i$ and we call each $x_{ij}$ as \emph{feature} in this paper. Assume $D=\sum_{m=1}^d R_m$ denotes the number of features and, as a reminder, $d$ is the number of raw features.

We next give a formal definition for open-world feature extrapolation problem in this paper: given training data $\{(\mathbf x_i, y_i)\}_{i\in I_{tr}}$ where $\mathbf x_i\in\mathcal X_{tr} = \{0, 1\}^D$, $y_i\in \mathcal Y$ and $I_{tr}$ is a set of indices, we aim to learn a model that can generalize to test data $\{\overline{\mathbf x}_{i'}, y_{i'}\}_{i'\in I_{te}}$ where $\overline{\mathbf x}_{i'}\in\mathcal X_{te} = \{0, 1\}^{\overline D}$, $y_{i'}\in\mathcal Y$ and $I_{te}$ is another set of indices. We term $\mathcal X_{tr}$ as \emph{training feature space} and $\mathcal X_{te}$ as \emph{test feature space}. We assume 1) the label space $\mathcal Y$ is shared by training and test data, and 2) $\mathcal X_{tr}\subset \mathcal X_{te}$, i.e., test feature space is an extension of training feature space. The feature space expansion stems from two possible causes: 1) there appear new raw features incrementally collected from other fields (i.e., $d$ increases) or 2) there appear new values out of the known space of existing raw features (i.e., $R_m$ increases).


\subsection{Proposed Model}\label{sec-model}

Our model contains three parts: 1) feature representation that builds a bipartite feature-data graph from input data; 2) a backbone network which is essentially a neural network model that predicts the labels when fed with input data; 3) a GNN model that inductively compute features' embeddings based on their proximity and local structures to achieve feature extrapolation.

\textbf{Feature Representation with Graphs.} We stack the feature vectors of all the training data as a matrix $\mathbf X_{tr} = [\mathbf x_{i}]_{i\in I_{tr}} \in \{0, 1\}^{N\times D}$ where $N=|I_{tr}|$. 
Then we treat each feature and instance as nodes and construct a bipartite graph between them. Formally, we define a node set $F_{tr}\cup I_{tr}$ where $F_{tr} = \{f_j\}_{j=1}^D$ with $f_j$ the $j$-th feature and $I_{tr} = \{o_i\}_{i=1}^N$ with $o_i$ the $i$-th instance in training set. The binary matrix $\mathbf X_{tr}$ constitutes an adjacency matrix where the non-zero entries indicate edges connecting two nodes in $F_{tr}$ and $I_{tr}$, respectively. The induced feature-data bipartite graph will play an important role in our extrapolation approach. The representation is flexible for variable-size feature set, enabling our model to handle test data $\overline{\mathbf x}_{i'}\in \{0,1\}^{\overline D}$ which gives $\mathbf X_{te} = [\overline{\mathbf x}_{i'}]_{i'\in I_{te}}$.

\textbf{Backbone Networks.} We next consider a prediction model $h_\theta(\cdot)$ as a backbone network that maps data features $\mathbf x_i$ to predicted label $\hat y_i$. Without loss of generality, a default choice for $h_\theta$ is a feedforward neural network. The first layer serves as an embedding layer which shrinks $\mathbf x_i$ into a $H$-dimensional hidden vector $\mathbf z_i = \mathbf x_i\mathbf W$ where $\mathbf W\in \mathbb R^{D\times H}$ denotes a weight matrix. The subsequent network (called \emph{classifier}) is often a stack of neural layers that predicts label $\hat y_i = \mbox{FFN}(\mathbf z_i; \phi)$. We use the notation $\hat y_i = h(\mathbf x_i; \phi, \mathbf W)$ to highlight two sets of parameters and $\theta = [\phi, \mathbf W]$.

Notice that the matrix multiplication in the embedding layer is equivalent to a two-step procedure: 1) a lookup of feature embeddings and 2) a permutation-invariant aggregation. More specifically, we consider $\mathbf W$ as a stack of weight vectors $\mathbf W = [\mathbf w_j]_{j=1}^D$ where $\mathbf w_j\in \mathbb R^{1\times H}$ corresponds to the embedding of feature $f_j$. The non-zero entries in $\mathbf x_i$ will index the corresponding rows of $\mathbf W$ and induce a set of embeddings $\{\mathbf z_i^m\}_{m=1}^d$ where $\mathbf z_i^m$ is the embedding given by $\mathbf x_i^m$ (i.e., one-hot vector of $m$-th raw feature). Then the hidden vector of $i$-th instance can be obtained by aggregation, i.e., $\mathbf z_i = \sum_{m=1}^d \mathbf z_i^m$ which is permutation-invariant w.r.t. the order of feature embeddings in $\{\mathbf z_i^m\}_{m=1}^d$. A more intuitive illustration is presented in Fig.~\ref{fig-backbone}.

The permutation-invariant property opens a way for handling variable-length feature vectors $\overline {\mathbf x}_{i'}$ \cite{deepset}. Essentially, on condition that we have embeddings for input features, we can add them up to get a fixed-dimensional hidden representation $\mathbf z_{i'}$ as input for the subsequent classifier. Therefore, the problem boils down to learning feature embeddings, especially \emph{how to extrapolate for embeddings of new features based on those of existing ones}. 

\emph{Remark.} Instead of using sum aggregation, some existing architectures consider concatenation of $d$ embedding vectors $\mathbf z_i^m$'s, which is essentially equivalent to sum aggregation (see Appendix~\ref{appx-eq} for more details). Therefore, the permutation-invariant property holds for widely adopted deep models.

\begin{figure}[t]
    \centering
    \includegraphics[width=\textwidth,angle=0]{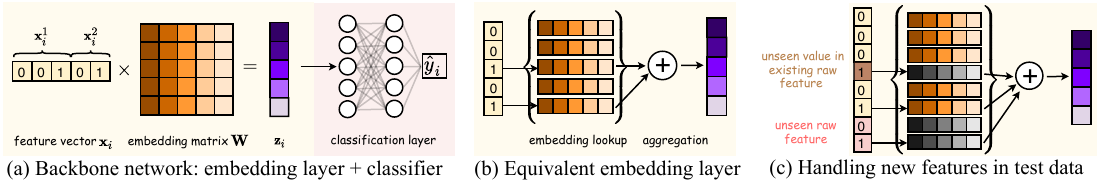}
    \caption{An illustration for the feasibility of extrapolation for new feature space with neural network models. (a) The backbone network is a feedforward neural network that is fed with input feature vector $\mathbf x_i$ and outputs predicted label $\hat y_i$. The first layer can be seen as an embedding layer where an embedding matrix is multiplied with the input vector to obtain an intermediate hidden representation $\mathbf z_i$. (b) The embedding layer can be equivalently replaced by an embedding lookup and an aggregation over indexed feature embeddings. The aggregation is permutation-invariant w.r.t. the order of input features. (c) The permutation-invariant property enables the model to handle variable-length input feature vector (new features may come from unseen values of existing raw features or unseen raw features). The problem boils down to learning new features' embeddings.}
    \label{fig-backbone}
\end{figure}

\textbf{GNN for Feature Extrapolation.} We proceed to propose a graph neural networks (GNN) model for embedding learning with the feature-data graph. Our key insight is that the bipartite graph explicitly embodies features' co-occurrence in observed instances, which reflects the proximity among features. Once we conduct message passing for feature embeddings over the graph structures, the embeddings of similar features can be leveraged to compute and update each feature's embedding. The model can learn to extrapolate for new features' embeddings using those of existing features with locality structures in a data-driven manner. The message passing over the defined graph representation is \emph{inductive} w.r.t. variable-sized feature nodes and instance nodes, which enables the model to tackle new feature space with distinct feature sizes and supports.

Specifically, we consider the embeddings $\mathbf w_j$ as an initial state $\mathbf w_j^{(0)}$ for node $f_j$ in $\mathcal F_{tr}$. The initial states of instance nodes are set as zero vectors with equal dimension as the feature nodes, i.e. $\mathbf s_i^{(0)} = \mathbf 0$. The interaction between two sets of nodes $\{\mathbf w_j\}_{j=1}^D$ and $\{\mathbf s_i\}_{i=1}^N$ can be modeled via graph neural networks where the node states in the $l$-th layer are updated by
\begin{equation}\label{eqn-gnn}
\begin{aligned}
    \mathbf s_i^{(l)} = \mathbf P^{(l)} \mbox{C\eqword{OMB}}\left (\mathbf s_i^{(l-1)}, \mbox{A\eqword{GG}}(\{\mathbf w_k^{(l-1)}\} | \forall k, x_{ik}=1) \right),\\
    \mathbf w_j^{(l)} = \mathbf P^{(l)} \mbox{C\eqword{OMB}}\left (\mathbf w_j^{(l-1)}, \mbox{A\eqword{GG}}(\{\mathbf s_k^{(l-1)}\} | \forall k, x_{jk}=1) \right),
\end{aligned}
\end{equation}
where $\mathbf P^{(l)} \in \mathbb R^{H\times H}$ is a weight matrix and we do not use non-linearity since it would degrade the performance empirically. For any new feature $f_{j'}$ in test data, we can set its initial state as a zero vector $\mathbf w_{j'}^{(0)}=\mathbf 0$. The GNN model outputs updated embeddings for feature nodes and we further use them as the feature embeddings in the backbone network. Fig.~\ref{fig-model} presents the feedforward computation of proposed model. Formally, with $L$-layer GNN, the GNN network $g$ gives updated feature embeddings $\hat {\mathbf W} = [\mathbf w_{j}^{(L)}]_{j=1}^D = g(\mathbf W, \mathbf X; \omega)$ and then the backbone network outputs prediction $\hat y_i = h(\mathbf x_i; \phi, \hat {\mathbf W})$ where $\mathbf X=\mathbf X_{tr}$ for training, $\mathbf X=\mathbf X_{te}$ for test and $\omega = \{\mathbf P^{(l)}\}_{l=1}^L$ denotes $g$'s parameters,.

\begin{figure}
    \centering
    \includegraphics[width=\textwidth,angle=0]{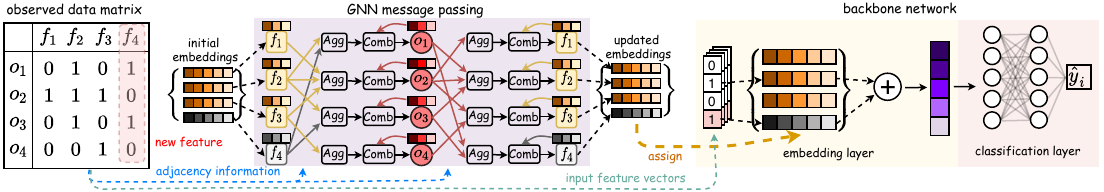}
    \vspace{-15pt}
    \caption{Feedforward of our model with input  $\{\mathbf x_i\}$: we build a feature-data graph between feature nodes $\{f_j\}$ and instance nodes $\{o_i\}$. A GNN is used to inductively compute features' embeddings via message passing and a backbone network uses the updated embeddings to predict labels $\{y_i\}$.  
    }
    \label{fig-model}
    \vspace{-10pt}
\end{figure}

\subsection{Model Learning}\label{sec-training}

We next discuss approaches for model training. 
In order to enable the model to extrapolate for new features, we put forward two useful strategies. 1) \emph{Proxy training data:} we only use partial features from training set as observed ones for each update. 2) \emph{Asynchronous updates:} we decouple the training of backbone network and GNN network and using different updating speeds for them in a nested manner (see Fig.~\ref{fig-training}). Based on these, we proceed to propose two specific training approaches.

\begin{figure}
    \centering
    \includegraphics[width=\textwidth,angle=0]{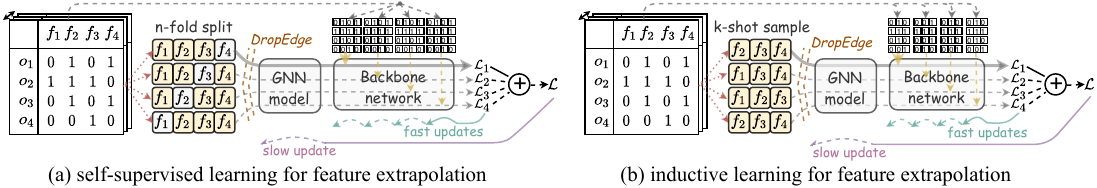}
    \vspace{-15pt}
    \caption{Illustration of two proposed training approaches: (a) self-supervised learning with n-fold splitting ($n=4$), and (b) inductive learning with k-shot sampling ($k=3$). The two methods differ in configuration of proxy data, and both consider asynchronous updating rule: the backbone network is updated with $n$-fold proxy data before the GNN network is updated once with an accumulated loss.
    }
    \label{fig-training}
    \vspace{-15pt}
\end{figure}

\textbf{Self-supervised Learning with N-fold Splits.} To mimic new features in the future, we can mask some observed features and let the model use the remaining features to estimate the embeddings of the masked ones. For a given feature set $F_{tr} = \{f_j\}_{j=1}^D$ of data $\{\mathbf x_i\}_{i\in I_{tr}}$, we consider an n-fold splitting method: in each iteration the features are first randomly shuffled and evenly divided into $n$ disjoint subsets, denoted by $\{\overline F_s\}_{s=1}^n$. We then consider asynchronous updating rule for two networks: each iteration contains $n$ times updates for backbone network and one update for GNN model. For the $s$-th update of the backbone, we mask features in $\overline F_s$ and set the initial states of masked features as zero vectors before fed into GNN. The GNN network will use the adjacency matrix $\mathbf X_{tr}$ to compute updated embeddings for the masked features. The embedding layer will be composed of updated embeddings of masked features and initial embeddings of the remaining features, based on which the backbone network outputs prediction $\hat y_i$ for each $\mathbf x_i$ and compute the loss function $\mathcal L_s = \frac{1}{N}\sum_{i\in I_{tr}} l(y_i, \hat y_i)$\footnote{Note that we still use supervised labels for loss though we call this approach as self-supervised learning.} (where $l(\cdot, \cdot)$ can be cross-entropy for classification). After $n$-step updates for backbone network, we use the accumulated loss $\mathcal L = \sum_{s}^n \mathcal L_s$ to update GNN model. The training procedure will repeat the above process until a given time budget. 

\textbf{Inductive Learning with K-shot Samples.} Alternatively, we can sample over the feature set and only expose partial features to the model for each update. For the $s$-th update, we randomly sample $k$ \emph{raw features}\footnote{One can also directly sample a certain ratio of features from $F_{tr}$, which might lead to large variance.} from input data which induces a new feature set $F_s\subset F_{tr}$ and extract the corresponding columns of $\mathbf X_{tr} = [\mathbf x_i]_{i\in I_{tr}}$ to form a proxy data matrix $\mathbf X_s\in \{0,1\}^{N\times |F_s|}$ (where each instance contains $|F_s|$ features). Then $\mathbf X_s$ is fed into GNN to obtain updated embeddings of features in $F_s$, based on which the backbone network outputs prediction for each instance using features in $F_s$. 

By contrast, the n-fold splitting contributes to better training stability since the model is updated on each feature in each iteration, while the inductive learning adds more randomness which can help to enhance model's generalization. We will further compare them in our experiments.

\textbf{DropEdge for Regularization.} In order to further alleviate over-fitting on training features, we use the DropEdge \cite{dropedge} to regularize our model. We consider a threshold $\rho$ and randomly set nonzero entries in $\mathbf X_{tr}$ (for self-supervised) or $\mathbf X_s$ (for inductive) as zero for each feedforward computation,
\begin{equation}
    \tilde{\mathbf X}_{tr} = \mbox{D\eqword{ROP}} \mbox{E\eqword{DGE}}(\mathbf X_{tr}, \rho) = \{x_{ij}|x_{ij}\in \mathbf X_{tr}, e_{ij}>\rho\}, \;\mbox{where}\; e_{ij}\sim \mathcal U(0,1).
\end{equation}

\textbf{Scaling to Large Systems.} To handle prohibitively large datasets for practical systems, we can divide data matrix into mini-batches along the instance dimension. Then, we feed each mini-batch into the model for once model training (including feature-level sampling/splitting, as shown in Fig.~\ref{fig-training}) or inference. 
Since the number of nonzero features for each instance is no more than $d$ (a relatively small value), the edge number in each mini-batch will be controlled within $O(Bd)$ (assume a mini-batch contains $B$ instances). Hence, the space cost can be effectively controlled using instance-level mini-batch partition. Yet, note that $B$ could not be arbitrarily small in order to guarantee sufficient message passing over diverse instances. We present the complete training algorithm in Appendix~\ref{appx-algo} where the model is trained end-to-end using self-supervised or inductive learning approaches.

\section{Generalization Analysis}\label{sec-gen}

In this section, we analyze the generalization error on test data with new features. We simplify the settings for analysis: 1) the backbone network is a two-layer FNN (an embedding layer $\mathbf W\in \mathbb R^{D\times H}$ plus a fully-connected layer $\Phi \in \mathbb R^{H\times 1}$) with sigmoid output; 2) the GNN network is a $L$-layer GCN which takes mean pooling aggregation over neighbored nodes without linear transformation and non-linearity in each layer; 3) the training algorithm is SGD. With above settings, the model can be written as $\hat y_i  = \sigma(\sum_{i'\in\mathcal N_{\tilde L}(i)\cup \{i\}} c_{ii'}^L\mathbf x_{i'} \mathbf W\Phi)$ where $\mathcal N_{\tilde L}(i)$ contains all the $\mathbf x_{i'}$'s that appear in the ${\tilde L}$-hop neighbors of $\mathbf x_i$ in the feature-data graph, $\tilde L = 2\cdot \lfloor \frac{L}{2}\rfloor$, and $c_{ii'}^L\in \mathbb R^+$ is a weight that quantifies influence of $\mathbf x_i$ on $\mathbf x_{i'}$ through $L$-layer mean-pooling graph convolution. More details for the derivation are in Appendix~\ref{appx-gen}. Also, we focus our analysis on the case of inductive learning and the results can be extended to self-supervised approach, which we leave for future work.

The data generation process can be described as follows. First, features $f_j$'s are sampled from an unknown distribution $\mathcal F$ and form a feature set $F_{tr} = \{f_j\}_{j=1}^D$. Then data $\{(\mathbf x_i, y_i)\}_{i\in I_{tr}}$ are sampled from a distribution $\mathcal D_{F_{tr}}$ whose support is over $\mathcal X_{F_{tr}}\times \mathcal Y$, and define $\mathbf X_{tr}=\{\mathbf x_i\}_{i\in I_{tr}}$ and $Y_{tr}=\{y_i\}_{i\in I_{tr}}$. Using $\psi = [\theta, \omega]$, the model can be denoted by $\hat Y = h(\mathbf X; \psi)$ (with a simplification from $\hat Y = \{\hat y_i \}_{i\in I} = \{ h(\mathbf x_i; \psi)\}_{i\in I}$) with loss function $\mathcal L(Y, \hat Y) = \frac{1}{|I|}\sum_{i\in I} l(y_i, \hat y_i)$. 

Recall that in training stage, we randomly partition the instances in $\mathbf X_{tr}$ into mini-batches with size $B$ and each mini-batch further samples $k$ raw features to form a feature subset $F_{s}$. In each update, the model is exposed to a $B\times |F_s|$ sub-matrix from $\mathbf X_{tr}$ as proxy training data and uses it for once feedforward and backward computation. With given training data $(\mathbf X_{tr}, Y_{tr})$, we define $\mathcal S$ as a set of all the proxy data sub-matrices that could be exposed to the model during the training
\[
    \mathcal S = \{(\mathbf X_{1}, Y_{1}), (\mathbf X_{2}, Y_{2}), \cdots, (\mathbf X_{m}, Y_{m}), \cdots, (\mathbf X_{M}, Y_{M})\}, \;\mbox{where}\; M\propto \mathcal O\left (\frac{d!}{(d-k)!k!} \right).
\]
The training process can be seen as a sequence of operations each of which samples a sub-matrix from $\mathcal S$ as proxy data in an i.i.d. manner and computes gradients for one SGD update (more discussions are in Appendix~\ref{appx-gen}).
Define $A_{\mathcal S}$ as a learning algorithm trained on $\mathcal S$, which gives a trained model $h(\mathbf X; \psi_{\mathcal S})$ simplified as $h_{\mathcal S}$. The generalization error $R(h_{\mathcal S})$ can be defined as
\begin{equation}
    R(h_{\mathcal S}) = \mathbb E_{(\mathbf X, Y)}\left[\mathcal L(Y, h(\mathbf X; \psi_{\mathcal S}))\right],
\end{equation}
where the expectation contains two stages of sampling: 1) a feature set $F=\{f_j\}$ is sampled according to $f_j\sim \mathcal F$, and 2) data $(\mathbf X, Y)$ is sampled according to $(\mathbf x_i, y_i)\sim \mathcal D_F$. The empirical risk that our approach optimizes with the training data would be
\begin{equation}
    R_{emp}(h_{\mathcal S}) = \frac{1}{M} \sum_{m=1}^M \mathcal L(Y_{m}, h(\mathbf X_{m}; \psi_{\mathcal S})).
\end{equation}
We study the expected generalization gap
\begin{equation}
    \mathbb E_{A}[R(h_{\mathcal S}) - R_{emp}(h_{\mathcal S})],
\end{equation}
where the expectation is taken over the randomness of $A_{\mathcal S}$ stemming from sampling for SGD updates.

We assume that the loss function $l(y_i, \hat y_i)$ is Lipschitz-continuous and smooth w.r.t. the model output $\hat y_i$. Concretely, we have 1) $|l(y, f(\cdot)) - l(y, f'(\cdot))| \leq \beta |f(\cdot) - f'(\cdot)|$ and 2) $|\nabla l(y, f(\cdot)) - \nabla l(y, f'(\cdot))| \leq \beta' |\nabla f(\cdot) - \nabla f'(\cdot)|$. Such condition can be satisfied by widely used loss functions such as cross-entropy and MSE. Then we have the following result (see Appendix~\ref{appx-gen} for proof).
\begin{theorem}\label{thm-gen-main}
    Assume the loss function is bounded by $l(y_i, \hat y_i)\leq \lambda$. For a learning algorithm trained on data $\{\mathbf X_{tr}, Y_{tr}\}$ with $T$ iterations of SGD updates, with probability at least $1-\delta$, we have
    \begin{equation}
        \mathbb E_{A}[R(h_{\mathcal S}) - R_{emp}(h_{\mathcal S})] \leq \mathcal O(\frac{d^T}{M}) + \left (\mathcal O(\frac{d^T}{M^{2}}) + \lambda \right) \sqrt{\frac{\log(1/\delta)}{2M}}.
    \end{equation}
\end{theorem}
The generalization gap depends on the number of raw features in training data and the size of $\mathcal S$. The latter is determined by configuration of proxy training data, particular, sampling over training features. If the sampling introduces more randomness (e.g. $k\approx d/2$), $\mathcal S$ would become larger, contributing to tighter gap. However, a large $\mathcal S$ would also lead to large variance in training and amplify optimization error. Therefore, there exists a trade-off w.r.t. how to sample/split observed features in training stage. Furthermore, the generalization gap also depends on $d$, and a larger $d$ would result in looser bound (since one often has $T>d$). This is because a larger $d$ would require to deal with more features. As the network becomes \emph{wider} and complex, it would be more prone for over-fitting.

\section{Experiments}\label{sec-experiment}

We apply our model FATE (for \underline{F}e\underline{AT}ure \underline{E}xtrapolation Networks) on real-world datasets. First, we consider six classification datasets from UCI Machine Learning Repository \cite{uci-dataset}: Gene, Protein, Robot, Drive, Calls and Github, as collected from domains like biology, engineering and social networks. The feature numbers are ranged from 219 to 4006 and the instance numbers vary from 1080 to 58509. We consider two large-scale datasets Avazu and Criteo from real-world online advertisement system whose goal is to predict the Click-Through Rate (CTR) of exposed advertisement to users. The two datasets have $\sim40$ million clicking/non-clicking records as instances and $\sim2$ million features. More dataset information and implementation details are in Appendix~\ref{appx-dataset} and \ref{appx-implement}, respectively.

\vspace{-2pt}
\subsection{Experiment on UCI Datasets} 

\begin{figure}[tb!]
    \centering
    \includegraphics[width=0.95\textwidth,angle=0]{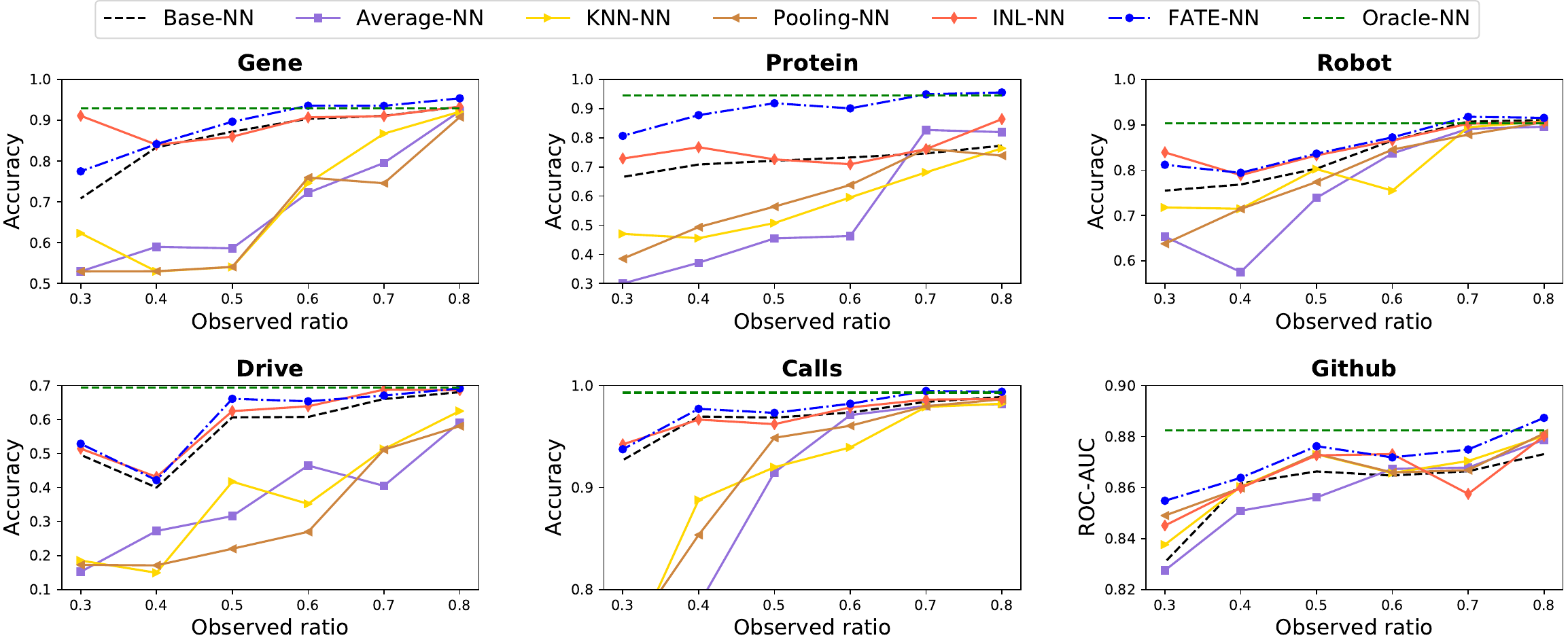}
   \vspace{-10pt}
    \caption{Accuracy/ROC-AUC results for UCI datasets with $30\%\sim80\%$ observed features for training. We run each experiment five times with different random seeds and report averaged scores.}
    \label{fig-uci-result}
    \vspace{-15pt}
\end{figure}

\textbf{Setup.} We randomly split all the instances into training/validation/test data with the ratio 6:2:2. Then we randomly select a certain ratio ($30\%\sim80\%$) of features as \emph{observed} ones and use the remaining as \emph{unobserved} ones. The model is trained with the \emph{observed} features of training instances and tested with \emph{all} the features of testing data. We adopt Accuracy as metric for datasets with more than two classes (Gene, Protein, Robot, Drive and Calls) and ROC-AUC for Github with two classes.

\textbf{Implementation.} We specify FATE in the following ways. 1) Backbone: a 3-layer feedforward NN. 2) GNN: a 4-layer GCN. 3) Training: self-supervised learning with n-fold splits. Several baselines are considered for comparison and their architectures are all specified as a 3-layer feedforward NN. First, \emph{Base-NN}, \emph{Average-NN}, \emph{Pooling-NN} and \emph{KNN-NN} are all trained on training instances with observed features. Then \emph{Base-NN} only uses test instances' observed features for inference. \emph{Average-NN} uses averaged embeddings of observed features as those of unobserved ones. \emph{Pooling-NN} (resp. \emph{KNN-NN}) computes embeddings for unobserved features via replacing our GNN with mean pooling aggregation over neighborhoods (resp. KNN aggregation over all the observed ones). Furthermore, we consider \emph{Oracle-NN} using all the features of training data for training and \emph{INL-NN} that is first trained on training data with observed features and then re-trained on training data with the remaining features.


\textbf{Results and Discussions.} Fig.~\ref{fig-uci-result} reports the mean Accuracy/ROC-AUC of five trials with different ratios of observed features ranging from 0.3 to 0.8. 
FATE achieves averagely $7.3\%$ higher Accuracy and $1.3\%$ higher ROC-AUC over Base-NN which uses partial features for inference. The improvements are statistically significant under $95\%$ confidence level. The results show that FATE can learn effective embeddings for new features that contribute to better performance for classification. Furthermore, FATE achieves averagely $29.8\%$ higher Accuracy and $0.5\%$ higher ROC-AUC over baselines Average-NN, KNN-NN and Pooling-NN. These baseline methods perform worse than Base-NN especially when observed features are few, which suggests that directly aggregating embeddings of observed features for extrapolation would degrade the performance. By contrast, FATE possesses superior capability for extrapolating to new unseen features. Even with $30\%$ observed features the model is able to distill the useful knowledge from $70\%$ unobserved features without re-training, providing decent classification performance. Notably, compared with INL-NN, FATE even achieves higher accuracy in most cases with a $4.1\%$ Accuracy improvement on average. The possible reason is that INL-NN is prone for over-fitting on new data and forgetting the previous one. Finally, FATE achieves very close performance to Oracle-NN when using sufficient observed features and can even slightly exceeds it with $80\%$ features in Gene, Robot and Github. In fact, the GNN network in FATE can not only achieve feature extrapolation, but also capture feature-level relations, which could be another merit of our method.

\subsection{Experiment on CTR Prediction}

\begin{table}[t!]
\small
\centering
\caption{ROC-AUC results for 8 test splits (T1-T8) on Avazu and Criteo datasets.}\vspace{-5pt}
\label{tbl-large-result}
\scalebox{0.83}{
\begin{tabular}{@{}c|c|c|cccccccccc@{}}
\toprule
Dataset & Backbone & Model & T1 & T2 & T3 & T4 & T5 & T6 & T7 & T8 & Overall \\ 
\midrule
\multirow{6}{*}{Avazu} & \multirow{3}{*}{NN} & Base & 0.666 & 0.680 & 0.691 & 0.694 & 0.699 & 0.703 & 0.705 & 0.705 & 0.693 $\pm$ 0.012  \\
& & Pooling & 0.655 & 0.671  & 0.683 & 0.683 & 0.689 & 0.694 & 0.697 & 0.697& 0.684 $\pm$ 0.011  \\
& & \textbf{FATE} & \textbf{0.689} & \textbf{0.699} & \textbf{0.708} & \textbf{0.710} & \textbf{0.715} & \textbf{0.720} & \textbf{0.721} & \textbf{0.721} & \textbf{0.710} $\pm$ 0.010
\\
\specialrule{0em}{1pt}{1pt}
\cline{2-12}
\specialrule{0em}{1pt}{1pt}
& \multirow{3}{*}{DeepFM} & Base & 0.675 & 0.684 & 0.694 & 0.697 & 0.699 & 0.706 & 0.708 & 0.706 & 0.697 $\pm$ 0.009  \\
& & Pooling & 0.666 & 0.676 & 0.685 & 0.685 & 0.688 & 0.693  & 0.694 & 0.694 & 0.685 $\pm$ 0.009  \\
&  & \textbf{FATE} & \textbf{0.692} & \textbf{0.702} & \textbf{0.711} & \textbf{0.714} & \textbf{0.718} & \textbf{0.722} & \textbf{0.724} & \textbf{0.724} & \textbf{0.713} $\pm$ 0.010 \\
\midrule
\multirow{6}{*}{Criteo} & \multirow{3}{*}{NN} & Base & 0.761 & 0.761  & 0.763  & 0.763 & 0.765 & 0.766 & 0.766 & 0.766 &  0.764 $\pm$ 0.002 \\
& & Pooling & 0.761 & 0.762 & 0.764 & 0.763 & 0.766 & 0.767 & 0.768 & 0.768 &  0.765 $\pm$ 0.001  \\
& &  \textbf{FATE} & \textbf{0.770} & \textbf{0.769} & \textbf{0.771} & \textbf{0.772}  & \textbf{0.773} & \textbf{0.774} & \textbf{0.774} & \textbf{0.774} & \textbf{0.772} $\pm$ 0.001 \\
\specialrule{0em}{1pt}{1pt}
\cline{2-12}
\specialrule{0em}{1pt}{1pt}
& \multirow{3}{*}{DeepFM} & Base & 0.772  & 0.771  & 0.772 & 0.772 & 0.774 & 0.774 & 0.774 & 0.774 &  0.773 $\pm$ 0.001
  \\
& & Pooling & 0.772 & 0.772 & 0.773 & 0.774 & 0.776  & 0.776 & 0.776 & 0.776 &  0.774 $\pm$ 0.002  \\
& & \textbf{FATE} & \textbf{0.781} & \textbf{0.780} & \textbf{0.782} & \textbf{0.782} & \textbf{0.784} & \textbf{0.784} & \textbf{0.784} & \textbf{0.784} & \textbf{0.783} $\pm$ 0.001
  \\
\bottomrule
\end{tabular}}
\vspace{-15pt}
\end{table}

\textbf{Setup.} We split the dataset in chronological order to simulate real-world cases. For Avazu dataset which contains time information in ten days, we use the data of first/second day for training/validation and the data of the third to tenth days for test. For Criteo dataset whose records are given in temporal order, we split the dataset into ten continual subsets with equal size and use the first/second subset for training/validation and the third to tenth subsets for test. With such data splitting, we can naturally obtain validation/test data with a mixture of seen and unseen features in the training data (the new features come from new values out of the known range of raw features). For Avazu/Criteo, there are $\sim$0.6/$\sim$1.3 million features in training data, $\sim$0.2/$\sim$0.4 million new features in validation data and totally $\sim$1.1/$\sim$0.8 million new features in all the test splits. We use ROC-AUC as evaluation metric.

\textbf{Implementation Details.} We consider two specifications for our backbone: 1) a 3-layer feedfoward NN, and 2) DeepFM \cite{deepfm}, a widely used model for advertisement click prediction considering inter-feature interactions over NN. The GNN is a 2-layer GraphSAGE model. The training method is inductive learning with k-shot samples and mini-batch partition. Also, we compare with baselines \emph{Base} and \emph{Pooling}. The \emph{KNN} method would suffer from scalability issue in the two large datasets.

\textbf{Results and Discussions.} Table~\ref{tbl-large-result} reports the ROC-AUC results for different test splits, which show that FATE significantly outperforms Base and Pooling in all the test splits using NN and DeepFM as backnones. Overall, compared with Base, FATE improves the ROC-AUC by 0.017/0.016 (resp. 0.008/0.01) with NN/DeepFM as the backbone on Avazu (resp. Criteo). Note that even an improvement of 0.004 for ROC-AUC is considered significant in click prediction tasks \cite{wd,deepfm}. The results show that FATE can combine useful information in new features collected in the future for the target task without further training and has promising power for enhancing the real-world systems interacting with open world. Compared with Pooling, FATE achieves significant AUC improvements on two datasets. The reason is that directly using average pooling to replace the GNN convolution would lead to limited capacity and weakens its ability for effective concept abstraction and reasoning.

\subsection{Further Discussions}

\textbf{Scalability.} We study model's scalability w.r.t. different batch sizes $B$ and number of features $D$ in Fig.~\ref{fig-scale-B} and \ref{fig-scale-D} which show that our training and inference time/space scale linearly on Criteo dataset.

\textbf{Ablation Studies.} Table~\ref{tbl-abl-uci} also provides ablation studies, which show that 1) the DropEdge operation can help regularize the training and bring up higher test accuracy; 2) using asynchronous updates for two networks leads to performance gain over joint training; 3) the n-fold splitting and k-shot sampling will outperform each other in different cases and both exceeds leave-one-out partition. Table~\ref{tbl-abl-large} compares using different sampling size $k$ for inductive learning approach on Avazu and Criteo, which verify our theoretical analysis in Section~\ref{sec-gen}. See more discussions in Appendix~\ref{appx-result}.

\textbf{Visualization.} Fig.~\ref{fig-vis} visualizes the produced feature embeddings by FATE-NN and Oracle-NN. It unveils two interesting insights that can interpret why FATE can sometimes outperform Oracle that uses all the features for training. First, FATE-NN's produced embeddings for observed and unobserved features possess more dissimilar distributions in latent space, compared to Oracle-NN. Notice that features' embeddings are used for the backbone network to compute intermediate hidden representation for each instance. Such phenomenon implies that FATE manages to extract more informative knowledge from new features. Second, the embeddings of FATE-NN form some particular structures (clusters, lines or curves) rather than uniformly distribute over the 2-D plane like Oracle-NN. The reason is that the GNN network leverages locality structures among features for further abstraction which explicitly encodes feature-level relations and could help downstream classification.

\begin{figure}[tb!]
    \centering
    \includegraphics[width=\textwidth,angle=0]{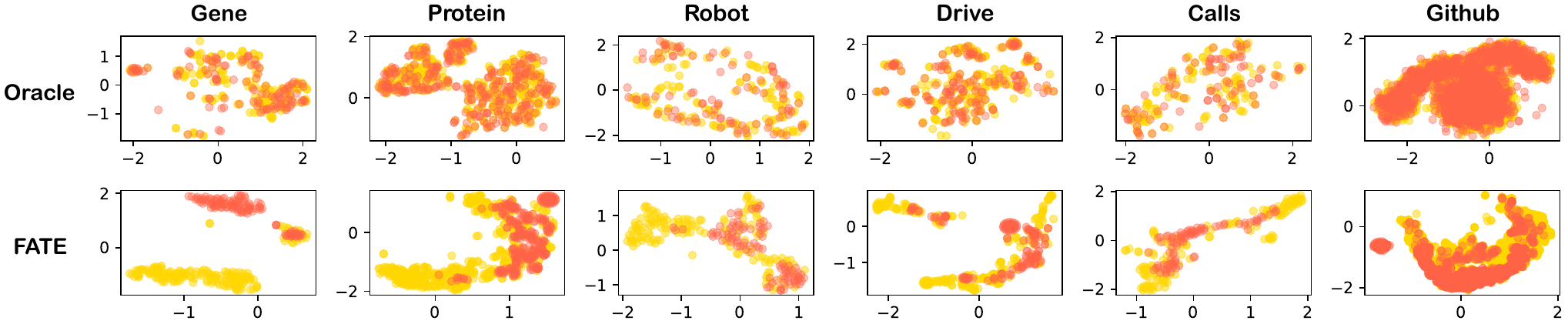}
    \vspace{-20pt}
    \caption{Visualization for t-SNE embeddings of FATE-NN's and Oracle-NN's produced feature embeddings. We mark observed features with red and the remaining ones with blue nodes.}
    \label{fig-vis}
    \vspace{-15pt}
\end{figure}

\section{Connection to Other Learning Paradigms}
Our introduced problem setting, open-world feature extrapolation (OFE), can be treated as an instantiation of out-of-distribution generalization \cite{ood-classic-2,ood-classic-3} or domain shift problem, focusing on distribution shift led by feature space expansion. We next discuss the relationships of our problem setting and our model FATE with domain adaption (DA), continual learning (CL), open-set learning (OSL) and zero-shot learning (ZSL). In general, OFE is orthogonal to these problems and opens a new direction that can potentially have promising intersections with the well-established ones.

\textbf{Domain Adaption} adapts a model trained on source domain to target with different distribution \cite{da-old1, da-old2, da-old3, da-nn1, da-nn2, da-nn3, da-nn4}. Our problem OFE is different from DA in two aspects: 1) the label space/distribution of training and test data is the same for OFE, while DA often mostly considers different label distributions for source and target domains; 2) OFE focus on combining new features that are related to the current task, while DA considers different tasks from different domains. 

\textbf{Continual Learning}, or lifelong learning, aims at enabling a single model to learn from a stream of data from different tasks that cannot be seen at one time \cite{cl-old1, cl-1, cl-2, cl-survey}. By contrast, there are two-fold differences of OFE. First, OFE does not allow finetuning or further retraining on new data, which can be more challenging than CL. Second, CL mostly assume each piece of data in the stream is from different tasks with different labels to handle. The key challenge of CL is the catastrophic forgetting \cite{cl-old1, cl-survey} that requires the model to balance a trade-off between previous and new tasks, while FATE is free from such issue in nature since we do not require incremental learning. 

\textbf{Open-set Learning} is another line of researches that relate with us. Differently, open-set recognition mostly focus on expansion of label sets \cite{openset-cv1, openset-cv2, openset-cv3, openset-cv4, openset-cv5, openset-nlp}. To our knowledge, we are the first to study feature set expansion, formulate it as OFE and further solve it via graph learning.


\textbf{Zero-shot Learning.}
Our problem setting is also linked with few-shot/zero-shot learning. In NLP domains, some studies focus on dealing with rare entities exposed in limited times or new entities unseen by training \cite{nlp-1,nlp-3}. Similar problems are also encoutered and explored in cold-start recommender systems where there are also new users/items unseen before \cite{recsys-1,recsys-2,recsys-3}. One common nature of these works aim at inferring the embeddings for new entities based on some `held-out' ones. 
With a similar end and distinct methodological aspect, a recent study \cite{recsys-3} explores a new possibility via learning a latent graph between existing entities (users) and newly arrived ones through attention mechanism for inductive representation learning. The core technical contributions of our work lie in the unique problem setting which stays focused on feature space expansion (domain shift) and the proposed feature-level sampling/partition training strategy, backed up with our theoretical insights.

\textbf{Extension and Outlooks.} Our work can be extended to solve more problems and push the development in broader areas. First, the input feature vectors can be replaced by feature maps given by CNN or word/sentence embeddings by Transformer \cite{transformer}, for handling multi-modal data in the context of federated learning \cite{mm-2} or multi-view learning \cite{mm-4,mm-6,mm-7}. Admittedly, our formulation assumes multi-hot feature vectors as input, which is a common practice for handling attribute features but not often the case for other data format (like vision or texts). For practitioners who would like to apply FATE to broader areas, one can extend our graph representation and GNN model to directly handle continuous features by treating feature values as edge weights as is done by \cite{graphlearning-missing}. Second, as shown in our experiments, the classification layers can be replaced by more complex models with inter-feature interactions \cite{ffm, xdeepfm, nfm, fwl} or advanced architectures \cite{danser, featevolve} for more sufficient feature-wise interaction and improving the expressiveness. 



\section{Conclusion}\label{sec-conclusion}

We present a new framework to address new features unseen in training, by formulating it as the open-world feature extrapolation problem. We target the problem via graph representation learning by treating observed data as a feature-data graph and further harness GNN to inductively compute embeddings for new features with those of existing ones, mimicking abstraction and reasoning process in human's brain. We also propose two training strategies for effective feature extrapolation learning. Our theoretical results show that generalization error depends on training features and learning algorithms. Experiments verify its effectiveness and scalability to large-scale systems.


\textbf{Potential Societal Impacts.} When learning mapping from features to labels, the model is at risk of focusing on dominant features from majority groups and ignoring scarce features from minority ones. Potential extended works of much significance are to develop debiased methods for feature extrapolation. We believe AI models can be guided to promote social justice and well-being.

\bibliographystyle{abbrv}
\bibliography{ref}

\clearpage

\clearpage
\appendix

\section{More Discussions for Permutation-Invariant Property of Embedding Layers}\label{appx-eq}

In Section~\ref{sec-model}, we mentioned that the embedding layer in the backbone network can be equivalently seen as a combination of embedding lookup and a sum aggregation which is permutation-invariant w.r.t. the order of input features. We provide an illustration for this in Fig.~\ref{fig-backbone}.

\textbf{Equivalence between Concatenation and Sum Aggregation.} To support the remark argument in Section~\ref{sec-model}, we next illustrate the equivalence between concatenation of features' embeddings and sum aggregation/pooling over features' embeddings. Assume we have feature embeddings $\{\mathbf z_i^m\}_{m=1}^d$ for instance $\mathbf x_i$. We concat all the embeddings as a vector $\overline {\mathbf z_i} = [\mathbf z_i^1, \mathbf z_i^2, \cdots, \mathbf z_i^m]$ and feed it into a neural layer to obtain $\mathbf c_i = \mathbf W' \overline{\mathbf z_i}$. Notice that the weight matrix $\mathbf W'\in \mathbb R^{dH\times H}$ can be decomposed into $d$ sub-matrices $\{\mathbf W'_m\}_{m=1}^d$ where $\mathbf W'_m = \mathbf W'[(m-1)H:mH,:] \in \mathbb R^{H\times H}$. If we consider sum aggregation/pooling over $\{\mathbf z_i^m\}_{m=1}^d$, i.e. $\mathbf z_i = \sum_{m=1}^d \mathbf z_i^m$, the subsequent neural layer would be a weight matrix $\mathbf W''$ with dimension $H\times H$. We can set it as $\mathbf W'' = \sum_{m=1}^d \mathbf W'_m$ and will easily obtain $\mathbf z_i \mathbf W'' = \overline{\mathbf z_i} \mathbf W' $. Hence, the concatenation plus a fully-connected layer is equivalent to sum pooling plus a fully-connected layer. This observation indicates that our reasoning in the maintext can be applied to general neural network-based models for attribute features and enable them to handle input vectors with variable-length features.


\section{Training Algorithms}\label{appx-algo}

We present the training algorithms for our model in Alg.~\ref{alg-training} where the model is trained end-to-end via self-supervised learning or inductive learning approaches.

\begin{algorithm}[h]
	\caption{Training algorithm for feature extrapolation networks (FATE).}
	\label{alg-training}
	\textbf{INPUT:} $\mathbf X_{tr}=\{x_i\}_{i\in I_{tr}}$, training data matrix, $F_{tr}=\{f_j\}$, training feature set, $\omega^{(0)}$, initial parameter for GNN, $\theta^{(0)}=[\phi^{(0)}, \mathbf W^{(0)}]$, initial parameter for backbone network (where $\phi^{(0)}$ denotes parameter for classifer and $\mathbf W^{(0)}$ denotes initial feature embeddings), $n$ split ratio, $k$ sample size, $\rho$ dropedge ratio, $\alpha_s, \alpha_f$, learning rates.\\
	{\For{$t=1,2,\cdots, T_{max}$}{
	Sample a mini-batch $\mathbf X^b=\{\mathbf x_i\}_{i\in I_b}$ from $\mathbf X_{tr}=\{\mathbf x_i\}_{i\in I_{tr}}$ // If handling large dataset, otherwise use $\mathbf X_{tr}$ directly \;
	
	\If{using self-supervised learning}{
	    Shuffle $F_{tr}=\{f_j\}$ and split into $n$ subsets $\{\overline F_s\}_{s=1}^n$ \;
    	{\For{$s=1,\cdots,n$}{
    	   $\mathbf W^{(t)}[f_j] \leftarrow \mathbf 0$, for $f_j\in \overline F_s$\;
    	   $\tilde{\mathbf X}^b = \mbox{D\eqword{ROP}} \mbox{E\eqword{DGE}}(\mathbf X^b, \rho)$\;
    	   Feed $\mathbf W^{(t)}$ and $\tilde{\mathbf X}^b$ into GNN and obtain $\hat {\mathbf W}^{(t)} = g(\mathbf W^{(t)}, \tilde{\mathbf X}^b; \omega^{(t)})$\;
    	   $\hat {\mathbf W^{(t)}}[f_j] = \mathbf W^{(t)}[f_j] $, for $f_j\in \mathcal F_{tr} \setminus \overline F_s$\;
    	   Feed $\{\mathbf x_i\}_{i\in I_b}$ into backbone network and obtain $\{\hat y_i\}_{i\in I_b}$ by $\hat y_i = h(\mathbf x_i;\phi^{(t)}, \hat {\mathbf W}^{(t)})$\;
    	   Compute the loss $\mathcal L_s(Y^b, \hat Y^b) = \frac{1}{|I_{tr}|} \sum_{i\in I{tr}} l(\hat y_i, y_i)$\;
    	   Update $\theta^{(t+1)}\leftarrow \theta^t - \alpha_f\nabla_{\theta} \mathcal L_s(Y^b, \hat Y^b)$\;
    	       }
    	}
    	Update $\omega^{(t+1)}\leftarrow \omega^t - \alpha_s \sum_{s=1}^n\nabla_{\omega} \mathcal L_s(Y^b, \hat Y^b)$\;
			}
		
		\If{using inductive learning}{
    	{\For{$s=1,\cdots,n$}{
    	Sample $k$ raw features, extract a subset $F_s$ from $F_{tr}=\{f_j\}$ and extract $\mathbf X^b_s$ from $\mathbf X^b$ \;
    	$\tilde{\mathbf X}^b_s = \mbox{D\eqword{ROP}} \mbox{E\eqword{DGE}}(\mathbf X^b_s, \rho)$\;
    	   Feed $\mathbf W^{(t)}$ and $\tilde{\mathbf X}^b_s$ into GNN and obtain $\hat {\mathbf W}^{(t)} = g(\mathbf W^{(t)}, \tilde{\mathbf X}^b_s; \omega^{(t)})$ \;
    	   Feed $\{\mathbf x_i\}_{i\in I_b}$ into backbone network and obtain $\{\hat y_i\}_{i\in I_b}$ by $\hat y_i = h(\mathbf x_i;\phi^{(t)}, \hat {\mathbf W}^{(t)})$\;
    	   Compute the loss $\mathcal L_s(Y^b, \hat Y^b) = \frac{1}{|I_{tr}|} \sum_{i\in I{tr}} l(\hat y_i, y_i)$\;
    	   $\theta^{(t+1)}\leftarrow \theta^t - \alpha_f\nabla_{\theta} \mathcal L_s(Y^b, \hat Y^b)$\;
    	       }
    	}
    	Update $\omega^{(t+1)}\leftarrow \omega^t - \alpha_s \sum_{s=1}^n\nabla_{\omega} \mathcal L_s(Y^b, \hat Y^b)$\;
			}
			
			}
			}
	\textbf{OUTPUT:} $\theta = [\phi, \mathbf W]$, trained parameter of backbone network, $\omega$, trained parameter of GNN.
\end{algorithm}

\section{Analysis of Generalization Error}\label{appx-gen}


We provide a complete discussion and proof for analysis on generalization error of our approach. Some notations are repeatedly defined in order for a self-contained presentation in this section. Recall that we focus our analysis on the case of inductive learning with k-shot sampling approach. Also, we simplify the model as :1) the backbone network is a two-layer FFN (an embedding layer $\mathbf W$ plus a fully-connected layer $\Phi$) with sigmoid output; 2) the GNN network is a $L$-layer GCN which takes mean pooling aggregation over neighbored nodes without linear transformation and non-linearity in each layer; 3) the training algorithm is SGD.

\textbf{Derivation for model function.} With our settings in Section~\ref{sec-gen}, we write the model as
\begin{equation}
    \hat y_i = h(\mathbf x_i; \phi, \hat{\mathbf W}) = h(\mathbf x_i; \phi, g(\mathbf W, \mathbf X; \omega)) = \sigma \left(\sum_{i'\in\mathcal N_{\tilde L}(i)\cup \{i\}} c^L_{ii'}\mathbf x_{i'} \mathbf W\Phi\right),
\end{equation}
where $\mathcal N_{\tilde L}(i)$ is a set which contains $\mathbf x_{i'}$'s that appear in the ${\tilde L}$-hop neighbors of $\mathbf x_i$ in the feature-data graph, $\tilde L = 2\cdot \lfloor \frac{L}{2}\rfloor$, and $c^L_{ii'}$ is a weight that quantifies influence of $\mathbf x_i$ on $\mathbf x_{i'}$ through $L$-layer mean-pooling graph convolution. Here we provide the detailed derivation. In fact, the embedding layer in the backbone can be seen as a one-layer GCN convolution using sum pooling without linear transformation and non-linearity, which can be denoted by $\hat Y = \sigma(\mathbf Z \Phi) = \sigma(\mathbf X \hat{\mathbf W} \Phi)$ where $\mathbf Z = \{\mathbf z_i\}$ (recall that $\mathbf X$ is treated as an adjacency matrix of the feature-data graph in our model). The GNN model, which is a $L$-layer GCN with mean pooling without linear transformation and non-linearity in each layer, can be denoted by $\hat {\mathbf W} = (\mathbf D_{out}^{-1}\mathbf X^\top \mathbf D_{in}^{-1} \mathbf X)^{\lfloor L/2\rfloor}\mathbf W$, where $\mathbf D_{in} = diag(\{d_{in, i}\}_{i=1}^N)$ with $d_{in, i} = \sum_{j\in F} x_{ij}$ and $\mathbf D_{out} = diag(\{d_{out, j}\}_{j=1}^F)$ with $d_{out, j} = \sum_{i\in I} x_{ij}$. Hence, we have $\mathbf Z = \mathbf X (\mathbf D_{out}^{-1}\mathbf X^\top \mathbf D_{in}^{-1} \mathbf X)^{\lfloor L/2\rfloor}\mathbf W = (\mathbf D_{out}^{-1} \mathbf D_{in}^{-1})^{\lfloor L/2\rfloor} \mathbf X (\mathbf X^\top \mathbf X)^{\lfloor L/2\rfloor} \mathbf W = (\mathbf D_{out}^{-1} \mathbf D_{in}^{-1})^{\lfloor L/2\rfloor} (\mathbf X \mathbf X^\top)^{\lfloor L/2\rfloor} \mathbf X \mathbf W$. Let $(\mathbf D_{out}^{-1} \mathbf D_{in}^{-1})^{\lfloor L/2\rfloor} (\mathbf X \mathbf X^\top)^{\lfloor L/2\rfloor} = \mathbf C_{L} \in \mathbb R^{N\times N}$ and $\mathbf C_{L} = \{c_{ii'}^L\}$ where $c_{ii'}^L$ denotes the a weight that quantifies influence between instance $i$ and $i'$ through $L$-layer GCN. Converting the global view of graph convolution into a local view for each node's ego-network, we can obtain $\hat y_i = \sigma(\mathbf z_i \Phi) = \sigma \left(\sum_{i'\in\mathcal N_{\tilde L}(i)\cup \{i\}} c^L_{ii'}\mathbf x_{i'} \mathbf W\Phi\right)$. 

With given training data $(\mathbf X_{tr}, Y_{tr})$, we define $\mathcal S$ as a set of all the data sub-matrices that could be sampled and exposed to the model during training 
\[
    \mathcal S = \{(\mathbf X_1, Y_1), (\mathbf X_2, Y_2), \cdots, (\mathbf X_m, Y_m), \cdots, (\mathbf X_M, Y_M)\}.
\]
The SGD training can be seen as a sequence of operations each of which picks an instance from $\mathcal S$ in an i.i.d. manner as a proxy training data and leverage it to compute updating gradient. We further introduce $\mathcal S^{\setminus m}$ which removes the $m$-th sub-matrix and $\mathcal S^{m}$ which replaces $m$-th sub-matrix by another one. Specifically, we have
\[
    \mathcal S^{\setminus m} = \{(\mathbf X_1, Y_1), \cdots, (\mathbf X_{m-1}, Y_{m-1}), (\mathbf X_{m+1}, Y_{m+1}) \cdots, (\mathbf X_M, Y_M)\},
\]
\[
    \mathcal S^m = \{(\mathbf X_{1}, Y_{1}), \cdots, (\mathbf X_{{m-1}}, Y_{{m-1}}), (\mathbf X_{{m}'}, Y_{{m}'}), (\mathbf X_{{m+1}}, Y_{{m+1}}) \cdots, (\mathbf X_{M}, Y_{M})\}.
\]

\textbf{Justification of the i.i.d. Sampling.} In fact, for our inductive learning approach in Section~\ref{sec-training}, the observed features for each proxy data are randomly sampled. The feature-level sampling at one time can be seen as $k$ times i.i.d. sampling from all the raw features without replacement. Denote $C=\{c_n\}_{n=1}^d$ as a set of all the raw features in training set, $\mathcal K$ as a set of $k$ distinct indices in $\{1,\cdots, d\}$ and $C_{\mathcal K}$ denotes a subset of raw features with indices from $\mathcal K$. Obviously, there are ${d \choose k}$ different configurations for $\mathcal K$ (or $C_{\mathcal K}$) in total. We can equivalently treat once feature-level sampling as a one-time i.i.d. sampling from a set of candidates $\{\mathcal K_m\}$ which contains ${d \choose k}$ index sets and each index set $\mathcal K_m$ contains $k$ indices from $\{1,\cdots, d\}$. Next we discuss two cases. 

1) If we do not consider instance-level mini-batch partition, then the set $\mathcal S$ will consist of $M = { d\choose k}$ sub-matrices. Specifically, the $m$-th sub-matrix $\mathbf X_m$ is induced by $C_{\mathcal K_m}$ which extracts the columns (corresponding to features generated by raw features in $C_{\mathcal K_m}$) of $\mathbf X_{tr}$. 

2) If we use instance-level mini-batch partition, the case would be a bit more complicated. First of all, the instance-level partition is not a strictly i.i.d. sampling process over training instances since their exists dependency among different mini-batches in one epoch. Yet, in practice, the batch size $B$ is very large (e.g. $B=100000$ in our experiment), so the number of mini-batches in one epoch is much smaller than $B$, which allows us to neglect the dependency in one epoch. Furthermore, since the instance-level selection is dependent of feature-level sampling, the whole sampling process for proxy data can be seen as a series of i.i.d. sampling over ${ N\choose B} \times { d\choose k}$ sub-matrices of $\mathbf X_{tr}$, which consists of the set $\mathcal S$ in this case.

Next, we recall the generalization gap of our interests. The generalization error $R(h_{\mathcal S})$ is defined as
\begin{equation}
    R(h_{\mathcal S}) = \mathbb E_{(\mathbf X, Y)}[\mathcal L(Y, h(\mathbf X; \psi_{\mathcal S}))].
\end{equation}
where the expectation contains two stages of sampling: 1) a feature set $F=\{f_j\}$ is sampled according to $f_j\sim \mathcal F$, and 2) data $(\mathbf X, Y)$ is sampled according to $(\mathbf x_i, y_i)\sim \mathcal D_F$. The empirical risk that our approach optimizes with the training data would be
\begin{equation}
    R_{emp}(h_{\mathcal S}) = \frac{1}{M} \sum_{m=1}^M \mathcal L(Y_{m}, h(\mathbf X_{m}; \psi_{\mathcal S})).
\end{equation}
Then the expected generalization gap would be
\begin{equation}
    \mathbb E_{A}[R(h_{\mathcal S}) - R_{emp}(h_{\mathcal S})],
\end{equation}
where the expectation is taken over the randomness of $A_{\mathcal S}$ that stems from sampling in SGD.

We next prove the result in Theorem~\ref{thm-gen-main} in our maintext. Our proof is based on algorithmic stability analysis \cite{stability-old}, following similar lines of reasoning in \cite{stability-old2,stability-kdd}. The main idea of the stability analysis is to bound the output difference of a loss function from a single data point perturbation. Differently, in our case, the `data point' is a data sub-matrix in $\mathcal S$. Therefore, our proof can be seen an extension of stability analysis to matrix data or graph as input. The proof can be divided into two parts. First, we derive a generalization error bound on condition of $\gamma$-uniform stability of the learning algorithm, Then we prove the bound for $\gamma$ based on our model architecture and SGD training. 

\subsection{Generalization error with uniform stability condition}

We first introduce the definition for uniform stability of a randomized learning algorithm as a building block of our proof. A randomized learning algorithm $A_{\mathcal S}$ is $\gamma$-uniform stable with regard to loss function $\mathcal L$ if it satisfies
\begin{equation}\label{eqn-stability1}
    \sup_{\mathcal S, (\mathbf X, Y)} |\mathbb E_{\mathcal S}[\mathcal L(Y, h(\mathbf X;\psi_{\mathcal S}))] - \mathbb E_{S}[\mathcal L(Y, h(\mathbf X;\psi_{\mathcal S^{\setminus m}}))]| \leq \gamma.
\end{equation}
We first prove a generalization bound using the uniform stability as a condition and then we prove that the learning algorithm in our case satisfies the condition. 
\begin{theorem}\label{thm-gen-con}
Assume a randomized algorithm $A_{\mathcal S}$ is $\gamma$-uniform stable with a bounded loss function $0\leq \mathcal L(Y, h(\mathbf X; \psi_{\mathcal S})) \leq L$. Then with probability at-least $1-\delta$ ($0<\delta<1$), over the random draw of $\mathcal S$, we have
\begin{equation}
    \mathbb E_{A}[R(h_{\mathcal S}) - R_{emp}(h_{\mathcal S})] \leq 2\cdot \gamma + (4M\gamma + L)\sqrt{\frac{\log\frac{1}{\delta}}{2M}}.
\end{equation}
\end{theorem}

\begin{proof}
Using triangle inequality, the stability property in \eqref{eqn-stability1} yields,
\begin{equation}
\begin{split}
    &\sup_{\mathcal S, (\mathbf X, Y)} |\mathbb E_{\mathcal S}[\mathcal L(Y, h(\mathbf X;\psi_{\mathcal S}))] - \mathbb E_{S}[\mathcal L(Y, h(\mathbf X;\psi_{\mathcal S^{m}}))]| \\
    \leq & \sup_{\mathcal S, (\mathbf X, Y)} |\mathbb E_{\mathcal S}[\mathcal L(Y, h(\mathbf X;\psi_{\mathcal S}))] - \mathbb E_{S}[\mathcal L(Y, h(\mathbf X;\psi_{\mathcal S^{\setminus m}}))]| \\
    + & \sup_{\mathcal S, (\mathbf X, Y)} |\mathbb E_{\mathcal S}[\mathcal L(Y, h(\mathbf X;\psi_{\mathcal S^m}))] - \mathbb E_{S}[\mathcal L(Y, h(\mathbf X;\psi_{\mathcal S^{\setminus m}}))]|\\
    \leq & \quad 2\gamma.
\end{split}
\end{equation}

We will use McDiarmid's concentration inequality for the following proof. Let $\mathbf Z$ be a random variable set and $f: \mathbf Z^M \rightarrow \mathbb R$. If it satisfies
\begin{equation}
    \sup_{z_1,\cdots, z_m, \cdots, z_M, z_{m}'} |f(z_1, \cdots, z_m, \cdots, z_M) - f(z_1, \cdots, z_m',\cdots, z_M)| \leq c_m,
\end{equation}
then we have
\begin{equation}\label{eqn-mcineq}
    P(f(z_1, \cdots, z_M) - \mathbb E_{z_1, \cdots, z_M}[f(z_1, \cdots, z_M)] \geq \epsilon) \leq \exp{\left(-\frac{2\epsilon^2}{\sum_{m=1}^M c_m^2}\right)}.
\end{equation}
Recall that data $(\mathbf X_{s}, Y_{s})$ are assumed to be i.i.d. sampled, so we have (assuming $\mathbf O = (\mathbf X, Y)$) 
\begin{equation}
    \begin{split}
        &\mathbb E_{\mathcal S}[\mathcal L(Y_{m}, h(\mathbf X_{m}; \psi_{\mathcal S}))]  \\
        = &
        \int \mathcal L(Y_{m}, h(\mathbf X_{m}; \psi_{\mathcal S})) p(\mathbf O_{1}, \cdots, \mathbf O_{M}) d \mathbf O_{1}\cdots d \mathbf O_{M} \\
        = & \int \mathcal L(Y_{m}, h(\mathbf X_{m}; \psi_{\mathcal S})) p(\mathbf O_{1})\cdots p(\mathbf O_{M})) d \mathbf O_{1}\cdots d \mathbf O_{M} \\
        = & \int \mathcal L(Y_{m'}, h(\mathbf X_{m'}; \psi_{\mathcal S^m})) p(\mathbf O_{1})\cdots p(\mathbf O_{m'})\cdots p(\mathbf O_{M}) d \mathbf O_{1}\cdots d \mathbf O_{m'} \cdots \mathbf O_{M} \\
        = & \int \mathcal L(Y_{m}, h(\mathbf X_{m}; \psi_{\mathcal S^m})) p(\mathbf O_{1}, \cdots, \mathbf O_{m'}, \cdots, \mathbf O_{M}) d \mathbf O_{1}\cdots d \mathbf O_{m'} \cdots d \mathbf O_{M} \times \int p(\mathbf O_{m}) d \mathbf O_{m} \\
        = & \int \mathcal L(Y_{m'}, h(\mathbf X_{m'}; \psi_{\mathcal S^m})) p(\mathbf O_{1}, \cdots, \mathbf O_{m'}, \mathbf O_{m}, \cdots, \mathbf O_{M}) d \mathbf O_{1}\cdots d \mathbf O_{M} d \mathbf O_{m'}\\
         = &\mathbb E_{\mathcal S, \mathbf O_{m'}}[\mathcal L(Y_{m'}, h(\mathbf X_{m'}; \psi_{\mathcal S^m}))].
    \end{split}
\end{equation}
Using above equation and the $\gamma$-uniform stability we have
\begin{equation}\label{eqn-proof-thm1-3}
    \begin{split}
        & \mathbb E_{\mathcal S}[\mathbb E_{A}[R(h_{\mathcal S}) - R_{emp}(h_{\mathcal S})]] \\
        = & \mathbb E_{\mathcal S}[\mathbb E_{\mathbf O}[\mathbb E_{A}[\mathcal L(Y, h(\mathbf X; \psi_{\mathcal S}))]]]
        - \frac{1}{M} \sum_{m=1}^M \mathbb E_{\mathcal S}[\mathbb E_{A}[ \mathcal L(Y_{m}, h(\mathbf X_{m}; \psi_{\mathcal S})) ]]\\
        = & \mathbb E_{\mathcal S}[\mathbb E_{\mathbf O}[\mathbb E_{A}[\mathcal L(Y, h(\mathbf X; \psi_{\mathcal S}))]]] 
        - \mathbb E_{\mathcal S}[\mathbb E_{A}[ \mathcal L(Y_{m}, h(\mathbf X_{m}; \psi_{\mathcal S})) ]]\\
        = & \mathbb E_{\mathcal S, \mathbf O_{{m}'}}[\mathbb E_{A}[ \mathcal L(Y_{m'}, h(\mathbf X_{m'}; \psi_{\mathcal S})) ]]
        - \mathbb E_{\mathcal S, \mathbf O_{{m}'}}[\mathbb E_{A}[ \mathcal L(Y_{m'}, h(\mathbf X_{m'}; \psi_{\mathcal S^m})) ]] \\
        = & \mathbb E_{\mathcal S, \mathbf O_{{m}'}}[ \mathbb E_{A}[\mathcal L(Y_{m'}, h(\mathbf X_{m'}; \psi_{\mathcal S})) - \mathcal L(Y_{m'}, h(\mathbf X_{m'}; \psi_{\mathcal S^m})) ] ] \\
        \leq & \mathbb E_{\mathcal S, \mathbf O_{{m}'}}[ \mathbb E_{A}[ | \mathcal L(Y_{m'}, h(\mathbf X_{m'}; \psi_{\mathcal S})) - \mathcal L(Y_{m'}, h(\mathbf X_{m'}; \psi_{\mathcal S^m})) | ] ] \\ 
        \leq & 2\gamma.
    \end{split}
\end{equation}
Also we have the following inequalities,
\begin{equation}\label{eqn-proof-thm1-1}
    \begin{split}
        |\mathbb E_{A}[R(h_{\mathcal S}) - R(h_{\mathcal S^m})]|  & = |\mathbb E_{\mathbf O}[\mathbb E_{A}[\mathcal L(Y, h(\mathbf X; \psi_{\mathcal S}))]] - \mathbb E_{\mathbf O}[\mathbb E_{A}[\mathcal L(Y, h(\mathbf X; \psi_{\mathcal S^m}))]]| \\
        & = |\mathbb E_{\mathbf O}[\mathbb E_{A}[\mathcal L(Y, h(\mathbf X; \psi_{\mathcal S}))]] - \mathbb E_{A}[\mathcal L(Y, h(\mathbf X; \psi_{\mathcal S^m}))]| \\
        & \leq \mathbb E_{\mathbf O}[\mathbb E_{A}[|\mathcal L(Y, h(\mathbf X; \psi_{\mathcal S})) - \mathcal L(Y, h(\mathbf X; \psi_{\mathcal S^m}))|]] \\
        & \leq 2\gamma,
    \end{split}
\end{equation}
\begin{equation}\label{eqn-proof-thm1-2}
    \begin{split}
        |\mathbb E_{A}[R_{emp}(h_{\mathcal S}) - R_{emp}(h_{\mathcal S^m})]|
        \leq &
        |\frac{1}{M} \sum_{m'=1, m'\neq m}^M  (\mathbb E_{A}[\mathcal L(Y_{{m'}}, h(\mathbf X_{{m'}}; \psi_{\mathcal S})) - \mathcal L(Y_{{m'}}, h(\mathbf X_{{m'}}; \psi_{\mathcal S^m}))])|\\
        +& |\frac{1}{M} \mathbb E_{A}[\mathcal L(Y_{{m}}, h(\mathbf X_{{m}}; \psi_{\mathcal S})) - \mathcal L(Y_{{m}'}, h(\mathbf X_{{m}'}; \psi_{\mathcal S^m}))]| \\
        \leq & 2\frac{M-1}{M} \gamma + \frac{\lambda}{M} \\
        \leq & 2\gamma + \frac{\lambda}{M}.
    \end{split}
\end{equation}
Letting $K_{\mathcal S} = R(h_{\mathcal S}) - R_{emp}(h_{\mathcal S})$ and using \eqref{eqn-proof-thm1-1} and \eqref{eqn-proof-thm1-2}, we obtain
\begin{equation}
    \begin{split}
        |\mathbb E_{A}[K_{\mathcal S}] - \mathbb E_{A}[K_{\mathcal S^m}]| & = \left | \mathbb E_{A}[R(h_{\mathcal S}) - R_{emp}(h_{\mathcal S})] - \mathbb E_{A}[R(h_{\mathcal S^m}) - R_{emp}(h_{\mathcal S^m})] \right | \\
        & \leq \left | \mathbb E_{A}[R(h_{\mathcal S})] - \mathbb E_{A}[R(h_{\mathcal S^m})] \right |  + \left | \mathbb E_{A}[R_{emp}(h_{\mathcal S})] - \mathbb E_{A}[R_{emp}(h_{\mathcal S^m})] \right |\\
        & \leq 2\gamma + (2\gamma + \frac{\lambda}{M}) \\
        & \leq 4\gamma + \frac{\lambda}{M}.
    \end{split}
\end{equation}
Based on the above fact, we can apply the result of \eqref{eqn-mcineq},
\begin{equation}
    P(\mathbb E_{A}[K_{\mathcal S}] - \mathbb E_{\mathcal S}[\mathbb E_{A}[K_{\mathcal S}]] \geq \epsilon) \leq \exp{\left (-\frac{2\epsilon^2}{M(4\gamma + \frac{\lambda}{M})^2}\right )}.
\end{equation}
Letting $\delta = \exp{(-\frac{2\epsilon^2}{M(4\gamma + \frac{\lambda}{M})^2})}$ and using \eqref{eqn-proof-thm1-3}, we obtain the following result and conclude the proof.
\begin{equation}
    P\left (\mathbb E_{A}[K_{\mathcal S}] \leq 2\gamma + (4M\gamma + \lambda)\sqrt{\frac{\log{(1/\delta)}}{2M}} \right) \geq 1 - \delta.
\end{equation}
\end{proof}

\subsection{Deriving bound for $\gamma$}

We proceed to prove our main result in Theorem~\ref{thm-gen-main} by deriving the bound for $\gamma$ based on the SGD algorithm and our GNN model. Let $\Theta_{\mathcal S}$ and $\Theta_{\mathcal S^m}$ denote the weight matrix of the classifier in the backbone network. Recall that our model is $\hat y_i = h(\mathbf x_i; \psi) = \sigma(\sum_{i'\in\mathcal N_{\tilde L}(i)\cup \{i\} } c_{ii'}^L\mathbf x_{i'} \mathbf W\Phi)$. Hence, we have
\begin{equation}\label{eqn-bound}
    \begin{split}
        & |\mathbb E_{SGD}[\mathcal L(Y, h(\mathbf X; \psi_{\mathcal S})) - \mathcal L(Y, h(\mathbf X; \psi_{\mathcal S^m}))] | \\
        = & \left|\mathbb E_{SGD}\left[ \frac{1}{N}\sum_{i\in I}l(y_i, h(\mathbf x_i; \psi_{\mathcal S})) - \frac{1}{N}\sum_{i\in I}l(y_i, h(\mathbf x_i; \psi_{\mathcal S^m}))\right] \right| \\
        \leq & \frac{\beta}{N}\mathbb E_{SGD}\left [\sum_{i\in I}|h(\mathbf x_i; \psi_{\mathcal S}) - h(\mathbf x_i; \psi_{\mathcal S^m})| \right] \quad(\mbox{since}\; l(\cdot, \cdot)\; \mbox{is} \; \beta\mbox{-Lipschitz})  \\
        = & \frac{\beta}{N}\mathbb E_{SGD}\left [\sum_{i\in I}\left | \sigma \left (\sum_{i'\in\mathcal N_{\tilde L}(i)\cup \{i\} } c_{ii'}^L\mathbf x_{i'} \mathbf W_{\mathcal S}\Phi_{\mathcal S} \right) - \sigma \left (\sum_{i'\in\mathcal N_L(i)\cup \{i\}} c_{ii'}^L\mathbf x_{i'} \mathbf W_{\mathcal S^m}\Phi_{\mathcal S^m}\right ) \right |\right]  \\
        \leq & \frac{\beta}{N}\mathbb E_{SGD}\left [\sum_{i\in I} \left | \sum_{i'\in\mathcal N_{\tilde L}(i)\cup \{i\} } c_{ii'}^L\mathbf x_{i'} \mathbf W_{\mathcal S}\Phi_{\mathcal S} - \sum_{i'\in\mathcal N_L(i)\cup \{i\}} c_{ii'}^L\mathbf x_{i'} \mathbf W_{\mathcal S^m}\Phi_{\mathcal S^m} \right | \right]  \\
        & (\mbox{due to the fact} \; |\sigma(x) - \sigma(y)| \leq |x-y|)\\
        \leq & \frac{\beta}{N} \mathbb E_{SGD} \left [\sum_{i\in I} \left \|\sum_{i'\in\mathcal N_{\tilde L}(i)\cup \{i\} } c_{ii'}^L\mathbf x_{i'}\right \|_2\cdot  \|\mathbf W_{\mathcal S}\Phi_{\mathcal S} - \mathbf W_{\mathcal S^m}\Phi_{\mathcal S^m} \|_2 \right ] \\
        \leq & \frac{\beta}{N} \sum_{i\in I}\left\|\sum_{i'\in\mathcal N_{\tilde L}(i)\cup \{i\} } c_{ii'}^L\mathbf x_{i'}\right \|_2 \mathbb E_{SGD}[\|\mathbf W_{\mathcal S}\Phi_{\mathcal S} - \mathbf W_{\mathcal S^m}\Phi_{\mathcal S^m}\|_2].
    \end{split}
\end{equation}
We need to bound the two terms in \eqref{eqn-bound}. First, notice that for $\forall \mathbf x_{i}$, it satisfies $\|\mathbf x_{i}\|_2 \leq \sqrt{d}$ and $\|x_{ij}\|_1 \leq d$ and the graph convolution with mean pooling induce the fact that $\|\sum_{i'\in\mathcal N_{\tilde L}(i)\cup \{i\} } c_{ii'}^L\mathbf x_{i'} \|_1 \leq d$. Using the inequality of arithmetic and geometric means, we have $\|\sum_{i'\in\mathcal N_{\tilde L}(i)\cup \{i\} } c^L_{ii'}\mathbf x_{i'}\|_2 \leq \sqrt{d}$.

We proceed to bound the second term by considering the randomness of SGD. We can define $\Psi_{\mathcal S} = \mathbf W_{\mathcal S}\Phi_{\mathcal S}$ as model parameters and we need to derive bound for $\mathbb E_{SGD}[\|\Psi_{\mathcal S} - \Psi_{\mathcal S^m}\|_2]$. Then define a sequence of model parameters $\{\Psi_{\mathcal S, 0}, \Psi_{\mathcal S, 1}, \cdots, \Psi_{\mathcal S, T}\}$ where $\Psi_{\mathcal S, t}$ denotes the model parameters learned by SGD on $\mathcal S$ with the updating in $t$-th step as
\begin{equation}\label{eqn-sgd}
    \Psi_{\mathcal S, t+1} = \Psi_{\mathcal S, t} - \alpha \nabla_{\Psi} \mathcal L(h(\mathbf X_{t}; \Psi_{\mathcal S, t}), Y_{t}) = \Psi_{\mathcal S, t} - \alpha \frac{1}{N_{t}} \sum_{i\in I_{t}} \nabla_{\Psi} l(h(\mathbf x_i; \Psi_{\mathcal S, t}), y_i).
\end{equation}
Similarly, $\{\Psi_{\mathcal S^m, 0}, \Psi_{\mathcal S^m, 1}, \cdots, \Psi_{\mathcal S^m, T}\}$ denotes a sequence of model parameters learned by SGD on $\mathcal S^m$. We then derive bound for $\Delta\Theta_t = \Psi_{\mathcal S, t} - \Psi_{\mathcal S^m, t}$ by considering two cases.

First, at step $t$, SGD picks data $(\mathbf X_{t}, Y_{t})$ and $t\neq m$, i.e., $(\mathbf X_{t}, Y_{t})$ exists in both $\mathcal S$ and $\mathcal S^m$. This case will happen with probability $\frac{M-1}{M}$. The derivative of model output is
\begin{equation}
    \frac{\partial h(\mathbf x_i; \Psi)}{\partial \Psi} = \sigma' \left(\sum_{i'\in\mathcal N_{\tilde L}(i)\cup \{i\} } c^L_{ii'}\mathbf x_{i'} \Psi \right) \cdot \sum_{i'\in\mathcal N_{\tilde L}(i)\cup \{i\} } c^L_{ii'}\mathbf x_{i'}.
\end{equation}
Using the fact $|\sigma'(x) - \sigma'(y)|\leq |\sigma(x) - \sigma(y)| \leq |x-y|$, we have
\begin{equation}\label{eqn-sgd-case1}
    \begin{split}
        & \|\nabla_{\Psi} \mathcal L(h(\mathbf X_{t}; \Psi_{\mathcal S, t}), Y_{t}) - \nabla_{\Psi} \mathcal L(h(\mathbf X_{t}; \Psi_{\mathcal S^m, t}), Y_{t})\|_2 \\
        \leq & \frac{1}{N_{t}} \sum_{i\in I_{t}} \|\nabla_{\Psi} l(h(\mathbf x_i; \Psi_{\mathcal S, t}), y_i) - \nabla_{\Psi} l(h(\mathbf x_i; \Psi_{\mathcal S^m, t}), y_i)\|_2 \\
        \leq &  \frac{\beta'}{N_{t}} \sum_{i\in I_{t}} \|\nabla_{\Psi}h(\mathbf x_i; \Psi_{\mathcal S, t}) -\nabla_{\Psi} h(\mathbf x_i; \Psi_{\mathcal S^m, t}) \|_2 \\
        \leq & \frac{\beta'}{N_{t}} \sum_{i\in I_{t}} \left \| \sigma' \left(\sum_{i'\in\mathcal N_{\tilde L}(i)\cup \{i\} } c^L_{ii'}\mathbf x_{i'} \Psi_{\mathcal S, t} \right) \cdot\sum_{i'\in\mathcal N_{\tilde L}(i)\cup \{i\} } c^L_{ii'}\mathbf x_{i'} - \sigma'\left (\sum_{i'\in\mathcal N_{\tilde L}(i)\cup \{i\} } c^L_{ii'}\mathbf x_{i'} \Psi_{\mathcal S^m, t}\right)\cdot \sum_{i'\in\mathcal N_{\tilde L}(i)\cup \{i\} } c^L_{ii'}\mathbf x_{i'} \right\|_2 \\
        \leq & \frac{\beta'}{N_{t}} \sum_{i\in I_{t}} \left\|\sum_{i'\in\mathcal N_{\tilde L}(i)\cup \{i\} } c^L_{ii'}\mathbf x_{i'}\right \|_2 \left \|\sigma'\left (\sum_{i'\in\mathcal N_{\tilde L}(i)\cup \{i\} } c^L_{ii'}\mathbf x_{i'} \Psi_{\mathcal S, t}\right) - \sigma' \left(\sum_{i'\in\mathcal N_{\tilde L}(i)\cup \{i\} } c^L_{ii'}\mathbf x_{i'} \Psi_{\mathcal S^m, t}\right)\right \|_2 \\
        \leq & \frac{\beta'}{N_{t}} \sum_{i\in I_{t}} \sqrt{d}\left \|\sum_{i'\in\mathcal N_{\tilde L}(i)\cup \{i\} } c^L_{ii'}\mathbf x_{i'} \Psi_{\mathcal S, t} - \sum_{i'\in\mathcal N_{\tilde L}(i)\cup \{i\} } c^L_{ii'}\mathbf x_{i'} \Psi_{\mathcal S^m, t}\right \|_2 \quad (\mbox{due to} \; |\sigma'(x) - \sigma'(y)|\leq |x-y| ) \\
        \leq & \frac{\beta'}{N_{t}} \sum_{i\in I_{t}} \sqrt{d} \left \|\sum_{i'\in\mathcal N_{\tilde L}(i)\cup \{i\} } c^L_{ii'}\mathbf x_{i'}\right \|_2 \|\Psi_{\mathcal S, t} - \Psi_{\mathcal S^m, t}\|_2 \\
        \leq & \beta' d \|\Delta \Psi_t\|_2.
    \end{split}
\end{equation}

Second, at step $t$, SGD picks $(\mathbf X_{t}, Y_{t})$ and $t = m$, i.e., $(\mathbf X_{t}, Y_{t})$ picked by the algorithm on $\mathcal S$ and $(\mathbf X'_{t}, Y'_{t})$ picked by the algorithm on $\mathcal S^m$ are distinct. This case would happen with probability $\frac{1}{M}$. We have
\begin{equation}\label{eqn-sgd-case2}
\begin{split}
& \|\nabla_{\Psi} \mathcal L(h(\mathbf X_{t}; \Psi_{\mathcal S, t}), Y_{t}) - \nabla_{\Psi} \mathcal L(h(\mathbf X'_{t}; \Psi_{\mathcal S^m, t}), Y'_{t})\|_2 \\
\leq & \frac{1}{N_{t}} \sum_{i\in I_{t}, j = I'_{t}[i]} \| \nabla_{\Psi} l(h(\mathbf x_i; \Psi_{\mathcal S, t}), y_i) - \nabla_{\Psi} l(h(\mathbf x_i'; \Psi_{\mathcal S^m, t}), y_i')\|_2 \\
& (\mbox{since}\; N_{t} = N'_{t}\; \mbox{and assume} \; I'_{t}[i]  \; \mbox{denotes the} \; i\mbox{-th entry in} \; I'_{t}) \\
\leq &  \frac{\beta'}{N_{t}} \sum_{i\in I_{t}, j = I'_{t}[i]}  \| \sigma' \left (\sum_{i'\in\mathcal N_{\tilde L}(i)\cup \{i\} } c^L_{ii'}\mathbf x_{i'} \Psi_{\mathcal S, t} \right )\cdot \sum_{i'\in\mathcal N_{\tilde L}(i)\cup \{i\} } c^L_{ii'}\mathbf x_{i'}  \\
- &  \sigma'\left (\sum_{i'\in\mathcal N_{\tilde L}(j)\cup \{j\} } c^L_{ji'}\mathbf x_{i'} \Psi_{\mathcal S^m, t} \right ) \cdot \sum_{i'\in\mathcal N_{\tilde L}(j)\cup \{j\} } c^L_{ji'}\mathbf x_{i'} \|_2 \\
\leq & \frac{\beta'}{N_{t}} \sum_{i\in I_{t}, j = I'_{t}[i]} \left \| \sigma'\left (\sum_{i'\in\mathcal N_{\tilde L}(j)\cup \{j\} } c^L_{ji'}\mathbf x_{i'} \Psi_{\mathcal S, t}\right ) \cdot \sum_{i'\in\mathcal N_{\tilde L}(j)\cup \{j\} } c^L_{ji'}\mathbf x_{i'} \right \|_2 \\
+ & \frac{\beta'}{N_{t}} \sum_{i\in I_{t}, j = I'_{t}[i]} \left \|\sigma' \left (\sum_{i'\in\mathcal N_{\tilde L}(j)\cup \{j\} } c^L_{ji'}\mathbf x_{i'} \Psi_{\mathcal S^m, t}\right ) \cdot \sum_{i'\in\mathcal N_{\tilde L}(j)\cup \{j\} } c^L_{ji'}\mathbf x_{i'} \right \|_2 \\
\leq & 2\beta' \sqrt{d}, 
\end{split}
\end{equation}
where the last inequality is due to $|\sigma'(x)|\leq 1$.

Combining \eqref{eqn-sgd-case1} and \eqref{eqn-sgd-case2} we have
\begin{equation}
\begin{split}
    & \mathbb E_{SGD}[\|\Delta\Psi_{t+1}\|_2] \\
    \leq & \frac{M-1}{M}  \mathbb E_{SGD} \left [ \left \|(\Psi_{\mathcal S, t} -  \alpha \nabla_{\Psi} \mathcal L(h(\mathbf X_{t}; \Psi_{\mathcal S, t}), Y_{t})) - (\Psi_{\mathcal S^m, t} -  \alpha \nabla_{\Psi} \mathcal L(h(\mathbf X_{t}; \Psi_{\mathcal S^m, t}), Y_{t})) \right \|_2 \right ]\\
    + & \frac{1}{M} \mathbb E_{SGD} \left [ \left \|(\Psi_{\mathcal S, t} -  \alpha \nabla_{\Psi} \mathcal L(h(\mathbf X_{t}; \Psi_{\mathcal S, t}), Y_{t})) - (\Psi_{\mathcal S^m, t} -  \alpha \nabla_{\Psi} \mathcal L(h(\mathbf X_{F'_t}; \Psi_{\mathcal S^m, t}), Y_{F'_t})) \right \|_2 \right ] \\
    \leq & \mathbb E_{SGD}[\|\Delta\Psi_t\|_2] + (1-\frac{1}{M}) \alpha \mathbb E_{SGD} \left [ \left \|\nabla_{\Psi} \mathcal L(h(\mathbf X_{t}; \Psi_{\mathcal S, t}), Y_{t}) - \nabla_{\Psi} \mathcal L(h(\mathbf X_{t}; \Psi_{\mathcal S^m, t}), Y_{t})\right \|_2 \right ] \\
    +  & \frac{1}{M} \alpha \mathbb E_{SGD} \left [ \left  \|\nabla_{\Psi} \mathcal L(h(\mathbf X_{t}; \Psi_{\mathcal S, t}), Y_{t}) - \nabla_{\Psi} \mathcal L(h(\mathbf X'_{t}; \Psi_{\mathcal S^m, t}), Y'_{t}) \right \|_2 \right ] \\
    = & \mathbb E_{SGD}[\|\Delta\Psi_t\|_2] + (1-\frac{1}{M}) \beta'd \mathbb E_{SGD}[\|\Delta\Psi_t\|_2] + \frac{2}{M} \beta' \sqrt{d} \\
    \leq & (1+\beta'd \mathbb E_{SGD}[|\Delta\Psi_t|] + \frac{2}{M} \beta' \sqrt{d}.
\end{split}
\end{equation}
The above inequality yields,
\begin{equation}
    \mathbb E_{SGD}[|\Delta\Psi_T|] \leq \frac{2\beta'\sqrt{d}}{M} \sum_{t=1}^T (1+\beta'd)^{t-1}.
\end{equation}
Plugging the result into \eqref{eqn-bound} we will obtain,
\begin{equation}
    \gamma \leq \frac{2\beta\beta'd}{M} \sum_{t=1}^T (1+\beta'd)^{t-1}.
\end{equation}
We complete the proof for Theorem~\ref{thm-gen-main}.

\section{Dataset Information}\label{appx-dataset}

\begin{table}[h]
\centering
\footnotesize
\caption{Information for experiment datasets. The Github dataset directly provides preprocessed 0-1 features.}
\label{tbl-ucidataset}
\begin{tabular}{@{}cccccccc@{}}
\toprule
Dataset & Domain & \#Instances & \#Raw Feat. & Cardinality  & \#0-1 Feat. & \#Class \\ \midrule
Gene & Life & 3190 & 60 & 4$\sim$6  & 287 & 3 \\
Protein & Life & 1080 & 80 & 2$\sim$8  & 743 & 8 \\
Robot & Computer & 5456 & 24 & 9  & 237 & 4 \\
Drive & Computer & 58509 & 49 & 9  & 378 & 11\\ 
Calls & Life & 7195 & 10 & 4$\sim$10  & 219 & 10 \\
Github & Social & 37700 & - & $\sim$  & 4006 & 2 \\
\hline
Avazu & Ad. & 40,428,967 & 22 & 5$\sim$1611749  & 2,018,025 & 2 \\
Criteo & Ad. & 45,840,617 & 39 & 5$\sim$541311  & 2,647,481 & 2 \\
\bottomrule
\end{tabular}
\end{table}

\subsection{Dataset Information}

We present detailed information for our used datasets concerning the data collection, preprocessing and statistic information. 

\textbf{UCI datasets.} The six datasets are provided by UCI Machine Learning repository \cite{uci-dataset}. They are from different domains, including biology, engineering and social networks. Gene dataset contains 60 DNA sequence elements, and the task is to recognize exon/intron boundaries of DNA. Protein dataset \cite{dataset-protein} consists of the expression levels of 77 proteins/protein modifications, genotype, treatment type and behavior, and the task is to identify subsets of proteins that are discriminant between eight classes of mice. Robot dataset is collected as a robot navigates through a room following a wall with 24 ultrasound sensor readings, and the task is to predict the robot behavior. Drive dataset is extracted from electric current drive signals with 49 attributes, and the task is to identify 11 different classes with different conditions. Calls dataset was created by segmenting audio records belonging to 4 different families, 8 genus, and 10 species, and the task is to identify the class of species. Github dataset \cite{dataset-github} is a large social network of GitHub developers with their location, repositories starred, employer and e-mail address, and the task is to predict whether the GitHub user is a web or a machine learning developer.

The six UCI datasets have diverse statistics. Overall, they contain thousands of instances and dozens of raw features with a mix up of categorial and continuous ones. The categorical raw features have cardinality ranged from 2 to 12. As mentioned in Section~\ref{sec-model}, the cardinality means the number of possible values for a discrete feature. For continuous features in each dataset (if exist), we first normalize the values into 0-mean and 1-standard-deviation distribution and then hash them into 10 buckets with evenly partition between the maximum and the minimum. Then each raw feature can be converted into one-hot representation. After converting all the features into binary ones we get up to hundreds of 0-1 features for each dataset. Table~\ref{tbl-ucidataset} summarizes the basic information for each dataset. 

\textbf{CTR prediction datasets.} The two click-through rate (CTR) prediction datasets have millions of instances and dozens of raw features with diverse cardinality. The goal of CTR prediction task is to estimate the probability that a user will click on an advertisement with the user's profile features and the ad's content features. In specific, Criteo\footnote{http://labs.criteo.com/2014/02/kaggle-display-advertising-challenge-dataset/} is a widely used public benchmark dataset for developing CTR models, which includes 45 million users’ click records, 13 continuous raw features and 26 categorical ones \footnote{In computational advertisement community, the raw features (e.g. site category, device id, device type, app domain, etc.) are often called \emph{fields}. We call them raw features in our paper to keep the notation self-contained.}. We follow \cite{ffm, feature-ctr} and use log transformation to convert the continuous features into discrete ones. Avazu\footnote{https://www.kaggle.com/c/avazu-ctr-prediction} is another publicly accessible dataset for CTR prediction, which contains users’ mobile behaviors including whether a displayed mobile ad is clicked by a user or not. It has 40 millions users' click records, 23 categorical raw feature spanning from user/device features to ad attributes (all are encoded to remove user identity information). The cardinality of different raw features for these two datasets are very diverse, ranging from 5 to a million. The raw features with very large cardinality include some id features, e.g. device id, site id, app id, etc. For each dataset, we convert each raw feature into one-hot representations and obtain 0-1 features. For features appearing less than 4 times we group them as one feature. After preprocessing, we obtain nearly 2 million 0-1 features for Avazu and Criteo as shown in Table~\ref{tbl-ucidataset}. 

\subsection{Dataset Splits}

\textbf{UCI datasets.} For each of UCI datasets, we first randomly partition all the instances into training/validation/test sets according to the ratio of 6:2:2. Then we randomly select a certain ratio ($30\%\sim 80\%$ in our experiments) of features as observed features and use the remaining as unobserved ones. The model is trained with observed features of training instances, validated with observed features of validation instances and tested with all the features of test instances.

\textbf{CTR prediction datasets.} As illustrated in Section~\ref{sec-experiment}, for Avazu/Criteo we split all the instances into ten folds in chronological order. Then we use the first fold for training, second fold for validation, and third to tenth folds for test. In such way, the validation data and test data will naturally contain new features not appeared in training data. Here we provide more illustration about this. As mentioned above, Avazu dataset contains 23 categorial raw features and some of them have very large cardinality. For example, the cardinality of raw features \emph{app id} and \emph{device id} are 5481 and 381763, respectively. In practical systems, there will be new apps introduced and new devices observed by the system as time goes by, and they play as new values out of the known range of existing raw features, which consist of new 0-1 features that are not unseen by the model (as introduced in the beginning of Section~\ref{sec-model}). Since we chronologically divide the dataset into training/validation/test sets, the validation and test sets would both contain a mixture of features seen in training set and new features unseen in training. Concretely, for Avazu, there are totally 618411 features in training set, 248614 new features (unseen by training data) in validation set, and totally 1151000 new features (unseen by both training and validation sets) in all the test sets. For Criteo, there are totally 1340248 features in training set, 472023 new features (unseen by training data) in validation set, and totally 835210 new features (unseen by both training and validation sets) in all the test sets. 

\section{Implementation Details}\label{appx-implement}

We present implementation details for our experiments for reproducibility. We implement our model as well as all the baselines with Python 3.8, Pytorch 1.7 and Pytorch Geometric 1.6.  The experiments are all run on a RTX 2080Ti, except for our scalability test in Section 4.3 where we use a RTX 8000.

\subsection{Details for UCI experiments}

\textbf{Architectures.}
For experiments on UCI datasets, the network architecture for our backbone network is
\begin{itemize}
    \item A three-layer neural network with hidden size 8 in each layer.
    \item The activation function is ReLU.
    \item The output layer is a softmax function for multi-class classification or sigmoid for two-class classification.
\end{itemize}
The architecture for our GNN network is
\begin{itemize}
    \item A four-layer GCN \cite{gcn} network with hidden size 8 in each layer.
    \item Adding self-loop and using normalization for graph convolution in each layer.
    \item No activation unit is used.
\end{itemize}

\textbf{Training Details.} We adopt self-supervised learning approach with n-fold splitting. Concretely, in each epoch, we feed the whole training data matrix into the model and randomly divide all the observed features into $n=5$ disjoint sets $\{\overline F_s\}_{s=1}^n$. Then a nested optimization is considered: 1) we update the backbone network with $n$ steps where in the $s$-th step, we mask observed features in $\overline F_s$; 2) then we update the GNN network with one step using the accumulated loss of the $n$ steps. The training procedure will repeat the above process until a given budget of 200 epochs. Also, in each epoch, the validation loss is averaged over $n$-fold data where for the $s$-th fold the features in $\overline F_s$ are masked and the model will use the remaining observed features for prediction. Finally, we report the test accuracy achieved by the epoch that gives the minimum logloss on validation dataset.

\textbf{Hyperparameters.}
Other hyper-parameters are searched with grid search on validation dataset. We use the same hyperparameter settings for six datasets, which indicates that our model is dataset agnostic in some senses. The settings and searching space are as follows:
\begin{itemize}
    \item The learning rates $\alpha_f$, $\alpha_s$ are searched within $[0.1, 0.01, 0.001]$. We set $\alpha_f=0.01$ and $\alpha_s=0.001$.
    \item The ratio for DropEdge $\rho$ is searched within $[0.0, 0.2, 0.5]$. We set $\rho=0.5$.
    \item The fold number for data partition $n$ is searched within $[2, 5, 10]$. We set $n=5$.
\end{itemize}

\textbf{Baselines.} All the baselines are implemented with a three-layer neural network, the same as the backbone network in our model. The baselines are all trained with a given budget of 200 epochs, and we report the test accuracy achieved by the epoch that gives the minimum logloss on validation dataset. The difference of them lies in the ways for leveraging observed and unobserved (new) features in training and inference. The detailed information for baseline methods is as follows.
\begin{itemize}
    \item \emph{Base-NN.} Use observed features of training instances for model training, and observed features of validation/test instances for model validation/test.
    \item \emph{Oracle-NN.} User all the features of training instances for model training, and all the features of validation/test instances for model validation/test.
    \item \emph{INL-NN.} The training process contains two stages. In the first stage, we train the model with initialized parameters using observed features of training instances for 200 epochs and save the model at the epoch that gives the minimum logloss on validation dataset (with observed features). In the second stage, we load the saved model in the first stage, train it using unobserved features of training instances for 200 epochs and report the test accuracy (using all the features) achieved by the epoch that gives the minimum logloss on validation dataset (using all the features).
    \item \emph{Average-NN.} Use observed features of training instances for model training. In test stage, we average the embeddings of observed features as the embeddings of unobserved features. Then the model would use all the features of test instances for inference (by using the trained embeddings of observed features and estimated embeddings of unobserved ones). 
    \item \emph{Pooling-NN.} Use observed features of training instances for model training. In test stage, we replace the GNN model in FATE with mean pooling over neighbored nodes. Specifically, the embeddings of unobserved features are obtained by non-parametric message passing using mean pooling over the feature-data bipartite graph.
    \item \emph{KNN-NN.} Use observed features of training instances for model training. In test stage, we compute the Jaccard similarity scores between any pair of observed and unobserved features. Then for each unobserved feature, its embedding is obtained by taking average of the embeddings of the observed features with top 20\% Jaccard similarities as the target unobserved feature.
\end{itemize}

\subsection{Details for Avazu/Criteo experiments}

\textbf{Architectures.}
For experiments on Criteo and Avazu datasets, we consider two specifications for the backbone network. First, we specify it as a feedforward NN, whose architecture is
\begin{itemize}
    \item A three-layer neural network with hidden size 10-400-400-1.
    \item The activation function is ReLU unit except the last layer using sigmoid.
    \item We use BatchNorm and Dropout with probability 0.5 in each layer.
\end{itemize}
Second, we specify it as DeepFM network \cite{deepfm}, which also contains an embedding layer and a subsequent classification layer. The embedding layer is an embedding lookup $\mathbf W$ which maps each nonzero index in $\mathbf x_i$ to an embedding, denoted as $\{\mathbf z_i^m\}$ where $\mathbf z_i^m$ denotes the embedding for the $m$-th raw feature of instance $i$. The subsequent classification layer can be denoted by
\begin{equation}
    \hat y_i = \sum_{j=1}^D w_j\cdot x_{ij} + \mbox{FM}(\{\mathbf z_i^m\}) + \mbox{FNN}(\mathbf z_i),
\end{equation}
where $\mathbf z_i = \sum_{m=1}^d \mathbf z_i^m$, FNN is a feedforward neural network and FM is a factorization machine which can be denoted as
\begin{equation}
    \mbox{FM} (\{\mathbf z_i^m\}) = \sum_{m, m'} <\mathbf z_i^m, \mathbf z_i^{m'}>.
\end{equation}
For our model FATE-DeepFM, we use the GNN model to compute feature embeddings $\hat {\mathbf W}$ based on which we use the input feature vector of an instance $\mathbf x_i$ to obtain $\{\mathbf z_i^m\}$ and $\mathbf z_i$ and then plug into the subsequent classification layer.

The architecture for our GNN network is
\begin{itemize}
    \item A two-layer GraphSAGE \cite{graphsage} network with hidden size 10 in each layer.
    \item No activation unit.
\end{itemize}

\textbf{Training Details.}
We adopt inductive learning approach with k-shot sampling for training our model. Furthermore, we use instance-level mini-batch partition to control space cost. Concretely, in each epoch, we first randomly shuffle all the training instances and partition them into mini-batches with size $B$. Then for each mini-batch, we consider a training iteration where the backbone network is updated with $n$ steps and the GNN model is updated with one step. For $s$-th step update for the backbone, we uniformly sample $k$ raw features from $d$ existing ones, obtain a new feature set $F_s$ induced by the sampled $k$ raw features, and extract the corresponding columns in the data matrix to form a proxy data, i.e., a $B\times |F_s|$ sub-matrix from $\mathbf X_{tr}$. We then use the proxy data matrix to update the backbone. After $n$-step updates for the backbone, we update the GNN model with the accumulated loss of the $n$ steps. 

All the models including FATE and baselines are trained with a given budget of 100 epochs. For every 10 training iteration, we compute the validation loss and ROC-AUC on validation dataset.
Finally, we report the test ROC-AUC achieved by the epoch that gives the highest ROC-AUC on validation dataset. 

\textbf{Hyperparameters.}
Other hyper-parameters are searched with grid search on validation dataset. The settings and searching space are as follows:
\begin{itemize}
    \item The learning rates $\alpha_f$, $\alpha_s$ are searched within $[0.1, 0.01, 0.001, 0.0001, 0.00001]$. We set $\alpha_f=0.0001$ and $\alpha_s=0.0001$ for NN as backbone. For DeepFM as backbone, we set $\alpha_f=0.0001$ and $\alpha_s=0.0001$ on Criteo, and $\alpha_s=0.00001$ on Avazu.
    \item The ratio for DropEdge $\rho$ is searched within $[0.0, 0.2, 0.5]$. We set $\rho=0.5$.
    \item The batch size $B$ is searched within $[10000, 20000, 100000, 200000]$. We set $B=100000$.
    \item The sampling size $k$ for data partition is searched within $[7, 11, 13, 17, 20]$ for Avazu and $[13, 16, 21, 24, 27]$ for Criteo. We set $k=17$ for Avazu and $k=24$ for Criteo.
\end{itemize}

\section{More Experiment Results}\label{appx-result}

We supplement more experiment results as further discussions of our method, including salability tests and ablation studies.

\subsection{Scalability Test}

We conduct experiments for scalability test on Criteo dataset. The scalability experiment is deployed on a RTX 8000 GPU with 48GB memory (though our comparison experiments in Section 4.1 and 4.2 require less than 12GB memory for each trial).

\textbf{Impact of Batch Sizes.} We statistic the running time per mini-batch for training and inference on Criteo dataset in Fig.~\ref{fig-scale-B}(a) and (b) where the batch size is changed from 1e5 to 1e6. The results are taken average over 20 mini-batches. As we can see, as the batch size increases, the training time and inference time both increase linearly, which depicts that our model has linear scalability w.r.t. the number of instances for each update and inference. Also, in Fig.~\ref{fig-scale-B}(c) and (d), we present the GPU memory cost for training and inference on Criteo dataset. As we can see, the space cost of FATE also increases linearly with respect with batch sizes. Indeed, as discussed in Section 3.2, the time and space complexity of FATE is $O(Bd)$ using mini-batch training where $d$ is relative small value (up to a hundred). Hence, the empirical results verify our analysis.

\textbf{Impact of Feature Numbers.} We also discuss the model's scalability concerning different feature numbers, i.e. the dimension of feature vectors $D$ for training data ($D'$ for test data). There are totally 39 raw features in Criteo dataset and we only use $[39, 36, 33, 30, 27, 24, 21, 18, 15]$ of them for experiments, which induces $[2647481, 2475154, 1956637, 1949494, 1466711, 927535, 862927, 850016]$ features, and also compare the training/inference time per mini-batch and GPU memory costs. The results are shown in Fig.~\ref{fig-scale-D}(a)-(d). We can see that as the feature number increases, the time and space costs both go up in linear trends, which indicates FATE has linear scalability w.r.t. feature number $D$. In fact, more feature numbers would require larger model size (for feature embeddings) and induce larger computational graph due to the increase of $d$; also, the increase of $d$ would also require more training/inference time based on our complexity analysis.

\subsection{Ablation Studies}

We next conduct ablation studies for some key components in our framework and discuss their impacts on our model. The results are shown in Table~\ref{tbl-abl-uci} and Table~\ref{tbl-abl-large}.

\textbf{Effectiveness of DropEdge.} In Table~\ref{tbl-abl-uci}, we compare with not using DropEdge regularization in training stage. The results show that FATE consistently achieve superior accuracy throughout six datasets, which demonstrate the effectiveness of DropEdge regularization that can help to alleviate the over-fitting on training features.

\textbf{Effectiveness of Asynchronous Updates.} We also compare our asynchronous updates (alternative fast updates for Backbone network and slow updates for GNN) with directly using end-to-end jointly training of two networks. The results show that FATE with asynchronous updates can outperform joint training approach over a large margin, which verify the effectiveness of our proposed asynchronous updating rule. The reason is that using asynchronous updates can decouple the training for two networks and further help two models learn useful information from observed data. Also, we observe that using slow updates for GNN network with the accumulated loss of several data splits can stabilize its training and alleviate the over-fitting.

\textbf{Comparison between Training Approaches.} We further investigate the k-fold splitting and n-fold sampling strategies used in our training approaches in Table~\ref{tbl-abl-uci}. Recall that in UCI datasets, we adopt the self-supervised learning approach for training. Here we compare our used n-fold splitting with leave-one-out, which leave out partial features as a fixed set for masking in training, and k-shot sampling, which randomly sample $\lfloor0.8*D\rfloor$ training features as observed ones and mask the remaining for each update. The results show that the n-fold splitting and k-shot sampling strategies both provide superior performance in six datasets. Furthermore, when using different $n$'s, the relative performance of n-fold splitting and k-shot sampling approaches diverge in different cases. Overall, we found using n-fold splitting with $n=5$ or $n=10$ work the best on average. In fact, the n-fold splitting and k-shot sampling both play as a role in mimicking new features and exposing partial observed features to the model in training. The difference is that n-fold splitting guarantees that in each iteration the model can be updated on each feature in training set while the k-shot sampling introduces more randomness. Unlike UCI datasets, in two large-scale datasets Criteo and Avazu where we adopt the inductive learning approach
for training, we found using k-shot sampling works consistently better than n-fold splitting. One possible reason is that k-shot sampling can increase the diversity of proxy data (containing partial features and partial instances) used for each training update and can presumably help the model to overcome feature-level over-fitting in large datasets. Such results are consistent with our theoretical generalization error analysis in Section~\ref{sec-gen}.

\textbf{Impact of Sampling Sizes.} We next study the impact of sampling size $k$ on the model performance. We use different $k$'s for inductive learning on Avazu and Criteo. The results are shown in Table~\ref{tbl-abl-large}. As we can see, as $k$ increases, the training AUC goes up, which demonstrates that larger sampling size can help for optimization since it reduces the variance of sampling and enhances training stability. Furthermore, it is not always beneficial to increase $k$. When it becomes large enough and close to the number of raw features $d$, the model would suffers from over-fitting. The results further demonstrate that large sampling size would lead to feature-level over-fitting, which echoes our theoretical results in Section~\ref{sec-gen}. Recall that Theorem~\ref{thm-gen-main} shows that model's generalization gap depends on the randomness in sampling over training features. Here when $k$ is large, there will be less randomness from feature-level data partition, which will degrade model's generalization ability. 



\clearpage
\begin{table}[t]
\centering
\caption{Ablation studies for DropEdge regularization, asynchronous updates for two networks (compared with end-to-end joint training) and our sampling strategies (n-fold splits and k-shot sampling compared with leave-one-out) on six UCI datasets. We run each experiment five times with different random seeds and report the mean scores.}
\label{tbl-abl-uci}
\begin{tabular}{@{}ccccccc@{}}
\toprule
Models & Gene & Protein & Robot & Drive & Calls & Github  \\ \midrule
w/o DropEdge & 0.9226 & 0.9031 & 0.8062 & 0.5261 & 0.9760 &0.8688 \\  
End-to-end Joint & 0.9257 & 0.8963 & 0.8454 & 0.1073 & 0.9762 & 0.7557 \\
\textbf{FATE (ours)} & \textbf{0.9345} & \textbf{0.9178} & \textbf{0.8815} & \textbf{0.6440} & \textbf{0.9839} & \textbf{0.8743} \\
\hline
Leave-one-out & 0.8564 & 0.6574 &  0.7641 & 0.4448 & 0.9334 & 0.8533\\
n-fold split ($n=10$) & 0.8884 & 0.8426 &  \textbf{0.8888} & 0.5910 & \textbf{0.9851} & 0.8723\\
\textbf{n-fold split ($n=5$)} & 0.9345 & \textbf{0.9178} & 0.8815 & \textbf{0.6440} & 0.9839 & 0.8743\\
n-fold split ($n=2$)  & 0.9298 & 0.8398 & 0.8359 & 0.5234  & 0.9514 & \textbf{0.8771}\\
k-shot sample ($n=10$) & \textbf{0.9404} &  0.9046 & 0.8839 & 0.5559 & 0.9812 & 0.8712\\
k-shot sample ($n=5$) & 0.9379 & 0.9102  & 0.8802 & 0.6060 & 0.9819 & 0.8712\\
k-shot sample ($n=1$) & 0.9304 &  0.8778 & 0.8568 & 0.5408 & 0.9611 & 0.8722\\
\bottomrule
\end{tabular}
\end{table}

\begin{table}[t]
\caption{Ablation studies for different sampling sizes $k$ for k-shot sampling in inductive learning on Avazu and Criteo. We report ROC-AUC on training data, validation data and 8-fold test data (T1-T8).}\vspace{-5pt}
\label{tbl-abl-large}
\scalebox{0.9}{
\begin{tabular}{@{}c|c|cccccccccc@{}}
\toprule
Dataset & $k$ & Train & Val & T1 & T2 & T3 & T4 & T5 & T6 & T7 & T8 \\ 
\midrule
\multirow{4}{*}{Avazu} & 11 & 0.7815 & 0.7369 & 0.6853 & 0.6950   & 0.7058 & 0.7093 & 0.7137 & 0.7186 & 0.7183 & 0.7193  \\
& 14 & 0.7842 & 0.7399 & 0.6896 & 0.6989 & 0.7080  & 0.7091 & 0.7142 & 0.7190 & 0.7201 & 0.7210    \\
& 17 & 0.7902 & 0.7433  & 0.6894 & 0.6995 & 0.7082 & 0.7105 & 0.7156 & 0.7203 & 0.7215 & 0.7216   \\
& 20 & 0.7978 & 0.7420 & 0.6872 & 0.6978 & 0.7080 & 0.7091  & 0.7146 & 0.7201 & 0.7201 & 0.7202  \\

\midrule
\multirow{5}{*}{Criteo} & 16 & 0.7955 & 0.7725 & 0.7669 & 0.7666 & 0.7688 & 0.7695 & 0.7714 & 0.7721 & 0.7722 & 0.7722                      \\
& 21 & 0.7988 & 0.7752 & 0.7699 & 0.7695 & 0.7714 & 0.7721 & 0.7736 & 0.7739 & 0.7741 & 0.7744                   \\
& 24 & 0.8005 & 0.7758 & 0.7701 & 0.7694 & 0.7712 & 0.7727  & 0.7732 & 0.7745 & 0.7740 & 0.7743                     \\
& 27 & 0.8025 & 0.7747 & 0.7698 & 0.7683 & 0.7711 & 0.7713 & 0.7727 & 0.7743 & 0.7734 & 0.7744                       \\
& 30 & 0.8057 & 0.7750 & 0.7690  & 0.7678 & 0.7701 & 0.7708 & 0.7725 & 0.7735 & 0.7723 & 0.7739 \\

\bottomrule
\end{tabular}}
\end{table}

\clearpage
\setcounter{figure}{5}  
\begin{figure}[t]
\begin{minipage}{0.9\linewidth}
\centering
\subfigure[Training time per mini-batch]{
\begin{minipage}[t]{0.48\linewidth}
\centering
\label{fig-ratio}
\includegraphics[width=0.98\textwidth,angle=0]{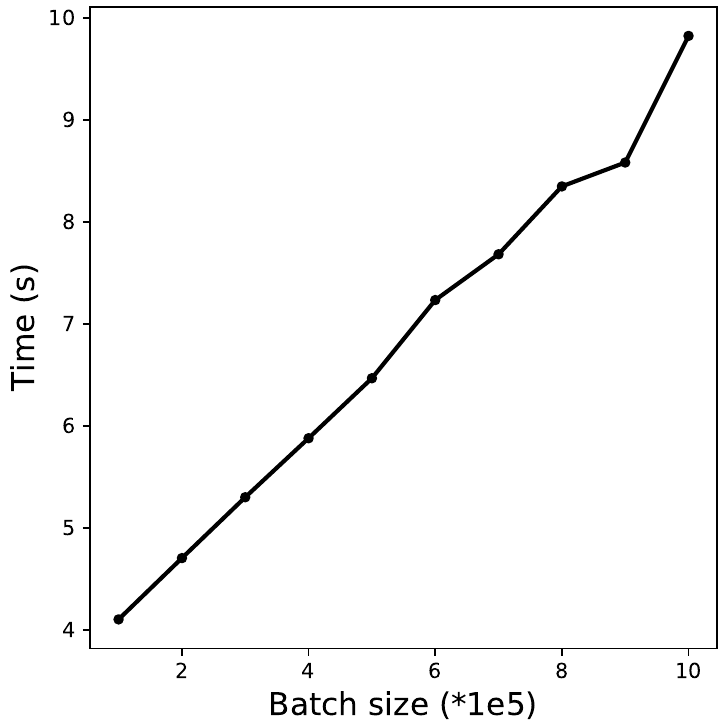}
\end{minipage}%
}%
\subfigure[Inference time per mini-batch]{
\begin{minipage}[t]{0.48\linewidth}
\centering
\label{fig-sparse}
\includegraphics[width=0.98\textwidth,angle=0]{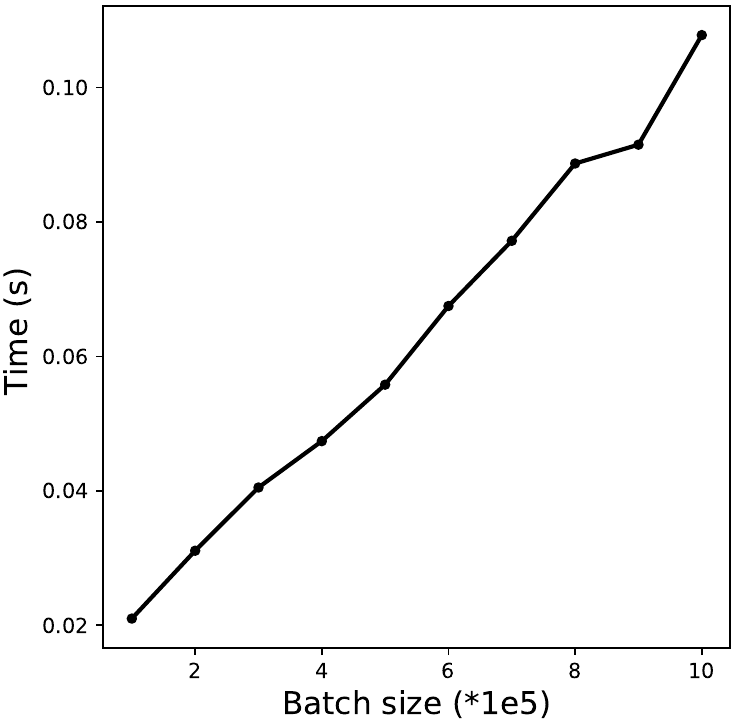}
\end{minipage}%
}%

\subfigure[GPU memory cost for training]{
\begin{minipage}[t]{0.48\linewidth}
\centering
\label{fig-attn-a}
\includegraphics[width=0.98\textwidth,angle=0]{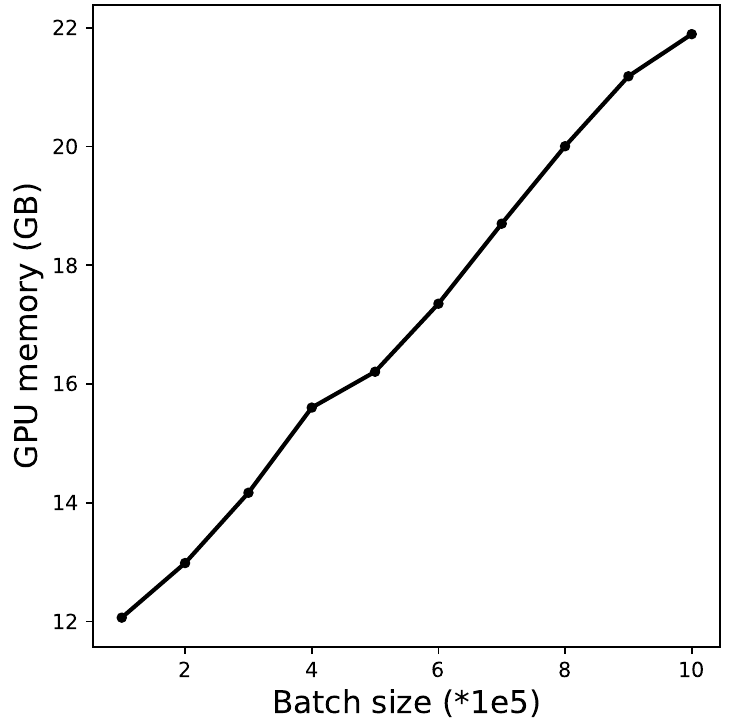}
\end{minipage}%
}
\subfigure[GPU memory cost for inference]{
\begin{minipage}[t]{0.48\linewidth}
\centering
\label{fig-attn-a}
\includegraphics[width=0.98\textwidth,angle=0]{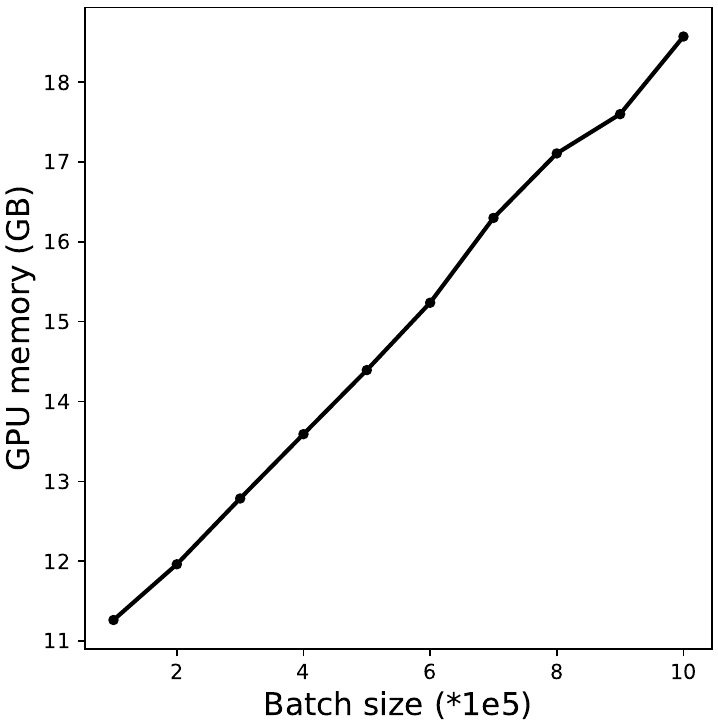}
\end{minipage}%
}
\end{minipage}
\caption{Scalability test of time and space costs w.r.t. batch sizes $B$ for training and inference on Criteo dataset. \label{fig-scale-B}}
\end{figure}

\newpage
\begin{figure}[t]
\begin{minipage}{0.9\linewidth}
\centering
\subfigure[Training time per mini-batch]{
\begin{minipage}[t]{0.48\linewidth}
\centering
\label{fig-ratio}
\includegraphics[width=0.98\textwidth,angle=0]{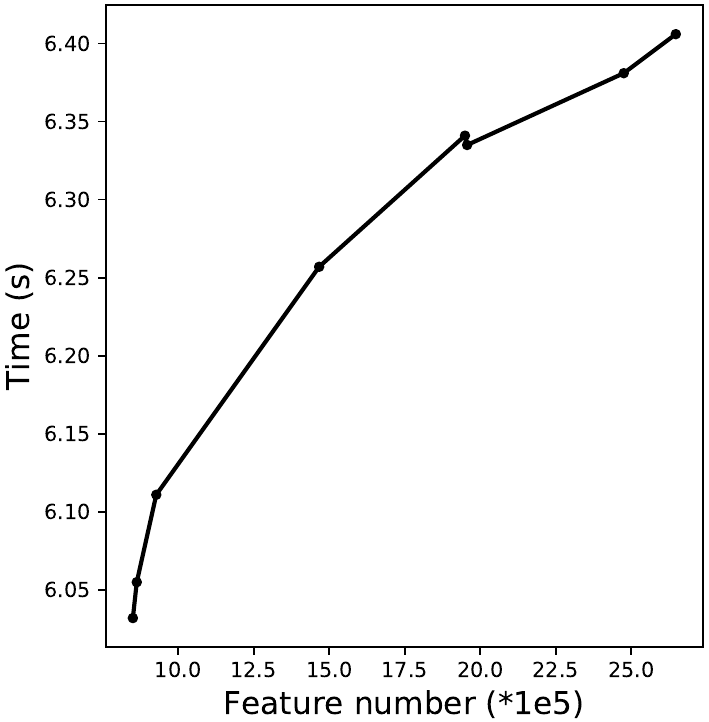}
\end{minipage}%
}%
\subfigure[Inference time per mini-batch]{
\begin{minipage}[t]{0.48\linewidth}
\centering
\label{fig-sparse}
\includegraphics[width=0.98\textwidth,angle=0]{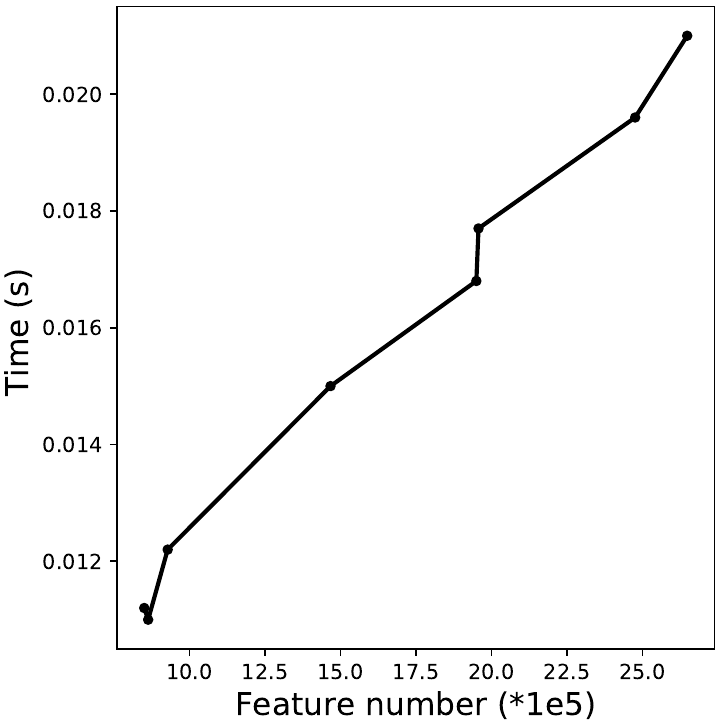}
\end{minipage}%
}%

\subfigure[GPU memory cost for training]{
\begin{minipage}[t]{0.48\linewidth}
\centering
\label{fig-attn-a}
\includegraphics[width=0.98\textwidth,angle=0]{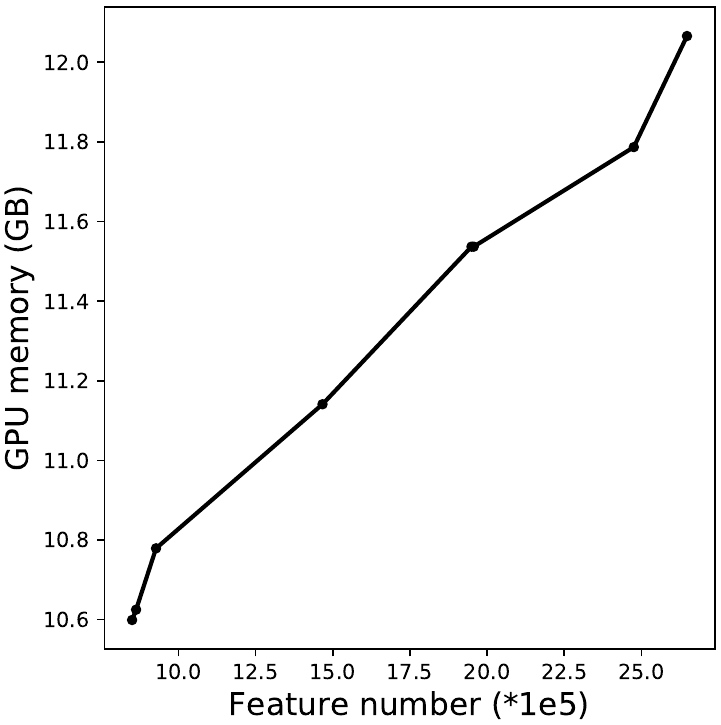}
\end{minipage}%
}
\subfigure[GPU memory cost for inference]{
\begin{minipage}[t]{0.48\linewidth}
\centering
\label{fig-attn-a}
\includegraphics[width=0.98\textwidth,angle=0]{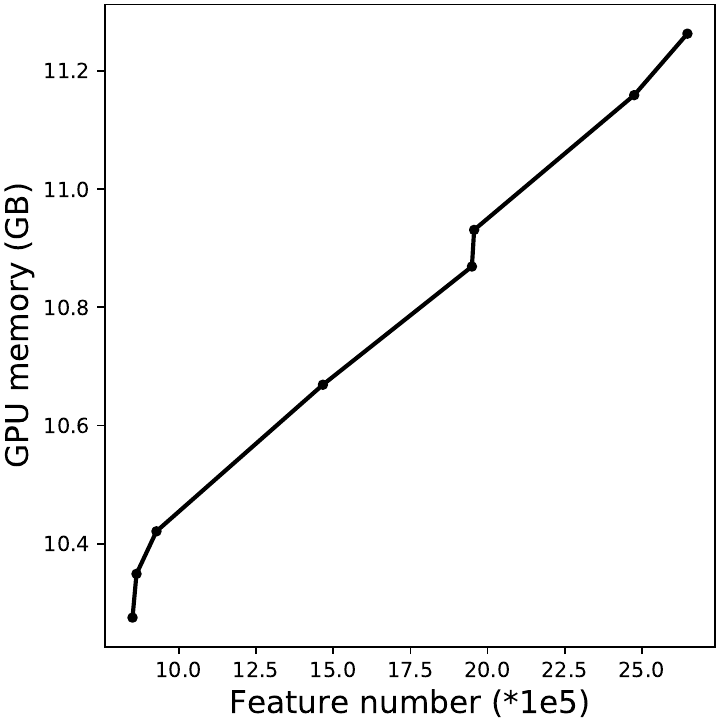}
\end{minipage}%
}
\end{minipage}
\caption{Scalability test of time and space costs w.r.t. feature numbers $D$ for training and inference on Criteo dataset. \label{fig-scale-D}}
\end{figure}

\end{document}